\documentclass[letterpaper,10 pt,journal,twoside]{IEEEtran}
\usepackage{times}

\usepackage{amssymb,amsthm}
\usepackage{graphicx}
\usepackage{amsmath}
\theoremstyle{definition}
\newtheorem{myDef}{Definition}

\newtheorem{myLemma}{Lemma}

\usepackage{bm}
\usepackage{algorithm}
\usepackage{multirow} 
\usepackage[dvipsnames]{xcolor}
\usepackage{algpseudocode}   
\usepackage{arydshln}
\newtheorem{theorem}{Theorem}
\usepackage{stfloats}
\newtheorem*{remark}{Remark}
\newtheorem{assumption}{Assumption}
\usepackage{subcaption}
\usepackage[numbers]{natbib}
\usepackage{multicol}
\usepackage[bookmarks=true]{hyperref}

\begin{document}

\title{Affine EKF: Exploring and Utilizing Sufficient and Necessary Conditions for Observability Maintenance to Improve EKF Consistency} 

\author{Yang Song$^{*}$, Liang Zhao$^{\dag}$, Shoudong Huang$^{*}$
\thanks{{Email: \tt\small Yang.Song-4@student.uts.edu.au}}
\thanks{$^{*}$Yang Song and Shoudong Huang are with Robotics Institute, University of Technology Sydney, Australia.}
\thanks{$^{\dag}$Liang Zhao is with  School of Informatics, The University of Edinburgh, United Kingdom.}
}



%

\maketitle

\begin{abstract}
Inconsistency issue is one crucial challenge for the performance of extended Kalman filter (EKF) based methods for state estimation problems, which is mainly affected by the discrepancy of observability between the EKF model and the underlying dynamic system. In this work, some sufficient and necessary conditions for observability maintenance are first proved. We find that under certain conditions, an EKF can naturally maintain correct observability if the corresponding linearization makes unobservable subspace independent of the state values. Based on this theoretical finding, a novel affine EKF (Aff-EKF) framework is proposed to overcome the inconsistency of standard EKF (Std-EKF) by affine transformations, which not only naturally satisfies the observability constraint but also has a clear design procedure. The advantages of our Aff-EKF framework over some commonly used methods are demonstrated through mathematical analyses.
The effectiveness of our proposed method is demonstrated on three simultaneous localization and mapping (SLAM) applications with different types of features, typical point features, point features on a horizontal plane and plane features. Specifically, following the proposed procedure,
the naturally consistent Aff-EKFs can be explicitly derived for these problems. The consistency improvement of these Aff-EKFs are validated by Monte Carlo simulations. The MATLAB code of the algorithm is made publicly available\footnote{\leftline{The MATLAB code is available at} \text{https://github.com/YangSONG-SLAM/Affine-EKF}}.


\end{abstract}

\IEEEpeerreviewmaketitle

\section{Introduction}
\IEEEPARstart{T}{he} state estimation is of great concern for nonlinear dynamic systems. Although offline optimization using all the measurements can produce very accurate estimates, the ever-growing size of the state hinders real-time computing \cite{Obj1,Zhuqing_IEKFlocalization}. Among the techniques for online computation, extended Kalman filter (EKF) is still one of the most popular algorithms in robotic society \cite{IEKF,past,InGVIO,FastLio2,PointLinePlane,Zhuqing_IEKFlocalization}. 



However, one crucial challenge for the performance of EKF is the inconsistency issue, namely the underestimation of uncertainty. 
 According to the studies on the inconsistency of EKF \cite{con2,con4,con1,con3}, the violation of observability constraint is a major factor causing such issue. For those partially-observable nonlinear dynamic systems, such as simultaneous localization and mapping (SLAM), standard EKF (Std-EKF) model may have less unobservable directions than the underlying nonlinear system. The spurious information gain along the missing unobservable directions will finally produce over-confident results \cite{FEJ2,con2}, leading to poor estimates.
 This limitation of Std-EKF has been shown not only in various SLAM tasks \cite{VINS1,ic_SLAM1,Obj1}, but also in other applications such as the state estimation for legged robots \cite{ic_LegRob}, the projectile trajectory estimation \cite{ic_ProjTraj}, and map-based localization \cite{Zhuqing_IEKFlocalization}. 


To address the inconsistency issue, there are currently two types of methods: one is represented by first estimated Jacobian EKF (FEJ-EKF) \cite{FEJ2,con2,DRI} or observability constrained EKF (OC-EKF) \cite{VINS1,con5,con4}; the other is represented by invariant EKF (I-EKF) \cite{con9,InGVIO,Obj1,con3}. The EKF models of the above methods maintain the correct observability property, i.e. they share the same dimension of observability as the underlying dynamic system. Specifically, to satisfy this observability constraint, the Jacobians in the FEJ-EKF and OC-EKF are not evaluated at the updated estimates, but the artificially adjusted values. This could result in poor performance in terms of accuracy and convergence \cite{FEJ2}. In contrast, I-EKFs employ the particular linearizations generated by specific Lie group structures so that they naturally satisfy the observability constraint. All the terms in I-EKFs are evaluated at the current updated estimates, rather than any artificially adjusted values, thus generally I-EKFs perform better than FEJ-EKFs and OC-EKFs. 

For any general nonlinear dynamic system, we hope we can find a linearization method similar to those in I-EKFs, based on which EKF naturally satisfies the observability constraint. However, recent I-EKF framework requires a special Lie group sturcture which may be difficult to find. It seems empirical to design those linearizations employed by I-EKFs on specific problems \cite{RIEKF_2DSLAM,InGVIO,ic_ProjTraj,Obj1,con3}. As far as we know, the existing literature did not provide a clear way to find such specific linearization for a general state estimation problem. 



This work aims to fill the gap by providing a clear procedure to design a theoretically sound observability preserved EKF for the state estimation problems. The idea of this paper stems from our finding that the observability discrepancy of Std-EKFs in some problems (such as \cite{Zhuqing_IEKFlocalization}) can be caused by the dependence of unobservable subspace on states. If this property is true for general problems, does it mean that as long as our designed linearization eliminates this dependence, the corresponding EKF can naturally satisfy the observability constraint? 

To answer this question, we further prove some sufficient and necessary conditions for observability maintenance from the perspective of dependence of unobservable subspace on states. 
Using these newly proved observability-based consistency theories, we develop a framework, termed affine EKF (Aff-EKF), which revamps inconsistent Std-EKF by performing specific affine transformations on its linearization to preserve the correct observability property. It is a general approach to transform any Std-EKF to naturally maintain correct observability, with the detailed technical procedure provided. More interestingly, by comparing the formulas of Aff-EKF and Std-EKF, we find those affine transformations on the linearization of Std-EKF are equivalent to correcting its covariance matrix calculation. This intuitively shows how Aff-EKF addresses the inconsistency issue of Std-EKF. 
Specifically, the main contributions of this paper include:
\begin{itemize}
    \item We study in-depth the conditions under which an EKF can naturally satisfy the observability constraint. To be specific, from the perspective of the relationship between the observability discrepancy of an EKF and the dependence of unobservable subspace on states, some sufficient and necessary conditions are proved.
    These newly proved theorems about observability can help us determine which linearization method can make the EKF satisfy the observability constraint.
    \item A novel framework, called Aff-EKF, is proposed to design observability-preserved estimators that revamp the inconsistent Std-EKFs through appropriate affine transformations. The generic procedure for designing such appropriate affine transformations is provided based on the proved sufficient and necessary conditions for observability maintenance. 
    \item We validate the effectiveness of our Aff-EKF framework on three SLAM applications with different types of features: typical point features, point features on a horizontal plane, and plane features. In addition, we also theoretically analyze the advantages of our proposed framework compared to the approaches in the literature. 
\end{itemize}

\section{Related Works}
\subsection{Inconsistency Issue}
The inconsistency issue of Std-EKF in SLAM problems was first discovered by Julier and Uhlmann \cite{AConterExample}. The paper \cite{AConterExample} shows that the (standard) EKF algorithm always yields an underestimated covariance for either a stationary or moving vehicle in the 2D point feature based SLAM. This inconsistency issue will be apparent after a long duration. 

Based on simulation results, Bailey \textit{et al}. \cite{Tim_Consistency} verified the inevitability of eventual inconsistency of Std-EKF. 
They found that simply reducing linearization errors through iterative EKF or unscented filters cannot fundamentally prevent inconsistency. However, they have realized that the issue can be addressed if Jacobians are evaluated at the ground-truth states. Huang \textit{et al}. \cite{con1} then theoretically proved that the inconsistency is caused by the violation of some fundamental constraints related to Jacobians, confirming the findings in \cite{Tim_Consistency}. 


Subsequently, the authors of \cite{con2,con4} further revealed that the inconsistency is caused by the discrepancy of observablity between Std-EKF and underlying dynamic system.  
They showed that, for 2D point feature based SLAM, the unobservable subspace of Std-EKF lacks the direction of robot orientation due to linearization process. This indicates that Std-EKF generated spurious information in such direction, which also explains the findings in \cite{Tim_Consistency}, and thus leading to the underestimated covariance. In \cite{con2,con4}, the proposed observability constraint is that a consistent EKF should share the same dimension of observability as the underlying dynamic system, which is a general criterion for treating the constraints proposed in \cite{con1,AConterExample} as special cases. 

The observability discrepancy of Std-EKF appears in not only the 2D point feature based SLAM but also many other robotic problems, such as legged robots \cite{ic_LegRob}, the projectile trajectory estimation \cite{ic_ProjTraj}, map-based localization \cite{Zhuqing_IEKFlocalization}, visual-inertial navigation systems (VINS) \cite{li2013high,RIEKF_VINS}, and SLAM with 3D point features \cite{con3} and object features \cite{Obj1}. 

{{Through the observability analyses, we notice that unobservable subspaces of inconsistent Std-EKFs for all the above problems are related to the state values, while that of naturally consistent I-EKFs are independent of the state values.}} 
Therefore, based on this finding, we believe there are some general conditions for observability maintenance from the perspective of the dependence of unobservable subspace on states. In this paper, we prove some
sufficient and necessary conditions of the EKF linearization for the observability constraint maintenance. 




\subsection{Consistent Estimators}
Based on the observability constraint \cite{con2,con4}, there are currently two kinds of approaches to design a consistent EKF. 

One is to artificially adjust the evaluation points of the Jacobians, forcing the EKF to retain the same dimension of observability as the underlying dynamic system. The representative methods include FEJ-EKF \cite{con2}, FEJ2-EKF \cite{FEJ2}, decoupled right invariant EKF (DRI-EKF) \cite{DRI} and observability constrained EKF (OC-EKF) \cite{VINS1,con5,con4}. FEJ-EKF \cite{con2} is the first consistent EKF in terms of observability constraint. It simply uses the first estimated values for Jacobian evaluations throughout the entire process, ensuring that the algorithm satisfies the desired observability property. Since FEJ-EKF depends heavily
on the initial estimates, Huang \textit{et al}. \cite{con4} proposed an alternative approach, OC-EKF, which evaluates Jacobians at the optimal values by minimizing the expected linearization errors under observability constraint. 

However, a critical defect of all such methods is that 
the first-order Taylor expansion will not be theoretically hold for the adjusted Jacobians, resulting in some non-negligible extra first-order linearization errors. FEJ2-EKF \cite{FEJ2}, a modification of FEJ-EKF, alleviates part of this limitation, but at the expense of some geometrical information loss. And numerical instability may occur during the implementation of FEJ2-EKF \cite{FEJ2}. DRI-EKF \cite{DRI} is another variant of FEJ-EKF. Instead of modifying $\mathbb{SO}$-EKF (defined in \cite{con3}, and regarded as Std-EKF in \cite{FEJ2,VINS1,con5,con2,con4}), it applies FEJ method to $\mathbb{SE}$-EKF (defined in \cite{con3}) for smaller linearization errors. Although the FEJ2-EKF and DRI-EKF outperform FEJ-EKF \cite{con2} and OC-EKF \cite{VINS1,con5,con4} in some experiments, both of them still have not fundamentally overcome the limitation of additional non-negligible first-order linearization errors.

In contrast, the second kind of approaches to address inconsistency issue are to design proper linearization methods \cite{con9,InGVIO,Obj1,RIEKF_VINS,con3}, such that the corresponding EKFs automatically satisfy the observability constraint. The linearization methods naturally follow the first-order Taylor expansion with no loss of any first-order information, which is more theoretically sound. Currently, these methods \cite{RIEKF_2DSLAM,con9,InGVIO,Obj1,RIEKF_VINS,con3} are originated from the framework of I-EKF \cite{IEKF}. The theories in \cite{IEKF,con9} reveal that if there is a Lie group structure such that the system is totally invariant under its group action, then the (invariant) EKF linearized by the corresponding Lie group logarithm will satisfy the observability constraint. In \cite{RIEKF_2DSLAM,con3}, the linearizations of I-EKFs for point feature based SLAM are generated by regarding the state space as the Lie group $\mathbb{SE}_{k+1}$. Similarly, Wu \textit{et al}. \cite{RIEKF_VINS} proposed I-EKF based on $\mathbb{SE}_{k+2}(3)\times \mathbb{R}^6$ for visual-inertial navigation systems (VINS). And Brossard \textit{et al}. \cite{con9} applied $\mathbb{SE}_{k+m}\times \mathbb{R}^q$ for  multi-sensor fusion. Also, we designed $\mathbb{SE}_{k+1}\times\mathbb{SO}^k$ for object based SLAM in our previous work \cite{Obj1}. 

While recent I-EKFs provide consistent estimations for several applications, they require appropriate Lie group structures that are compatible with the systems, which may not be readily available for general estimation problems \cite{PointLinePlane,Zhuqing_IEKFlocalization}. In this paper, we will exploit the newly proved theorems to develop a novel framework to design consistent EKFs for general problems without the need of finding specific Lie group structures.

\section{Preliminary Knowledge}
In this section, some preliminary knowledge about manifold is first introduced. Then, using the concepts of manifold, the general EKF procedure is presented. Finally, the observabilities of the EKF model and the underlying dynamic system are clearly defined.

\subsection{Manifolds}\label{PK_manifold}
In this paper, we assume the state space, $\mathcal M$, of the dynamic system under consideration is a manifold. 
There are different definitions of manifold, and here we follow \cite[page 4]{Lee2012} where definition of a manifold is given by the concepts of charts and atlases. 

 \begin{myDef}
A chart $\phi_{\textbf{X}}$ centered at $\textbf{X}\in \mathcal M$ is a pair $(\mathcal U_\textbf{X},\phi_\textbf{X})$\footnote{For brevity, in many places in this paper, we simply represent the pair $(\mathcal U_\textbf{X},\phi_\textbf{X})$ by $\phi_\textbf{X}$.}, where $\phi_\textbf{X}$ is an invertible mapping of $\mathcal (\textbf{X}\in)\ \mathcal{U}_\textbf{X}\subset \mathcal M$ onto an open set $\phi_\textbf{X}(\mathcal{U}_\textbf{X})\subset \mathbb{R}^m$ s.t. $\phi_\textbf{X}(\textbf{X})=\textbf{0}$. Here $m$ is the dimension of the manifold.
 \end{myDef}

 \begin{myDef}
 An atlas $\mathcal A=\{\phi_\textbf{X}|\textbf{X}\in \mathcal M\}$ is a collection of charts in which the transformations $\phi_{\textbf{X}_a}\circ\phi^{-1}_{\textbf{X}_b}$ are differentiable on $\mathcal{U}_{\textbf{X}_a}\cap\mathcal{U}_{\textbf{X}_b}$ for any $\textbf{X}_a$ and $\textbf{X}_b$.
 \end{myDef}

We simply define a set $\mathcal M$ to be a manifold if there is an atlas $\mathcal A$ on it\footnote{More rigorous mathematical definitions can be found in \cite{Manifoldbook,Lee2012}.}. Through an atlas, the ``intractable" m-dimensional manifold $\mathcal M$ is locally transformed to the ``well-known" Euclidean space $\mathbb{R}^m$. 

As an example, consider the product manifold $$\mathcal M=\mathbb{SO}(d)\times \mathbb{R}^d\times\mathbb{R}^{q},$$
we can naturally find an atlas $\mathcal A^\eta=\{\phi_{\hat{\textbf{X}}}\}$ on it, where
\begin{equation}\label{example_std}
    \phi_{\hat{\textbf{X}}}(\textbf{X})=(\text{log}^{\mathbb{SO}(d)}(\textbf{R}\hat{\textbf{R}}^\top),(\textbf{p}^r-\hat{\textbf{p}}^r)^\top,(\textbf{f}-\hat{\textbf{f}})^\top)\in \mathbb{R}^{2d+q}
\end{equation}
centered at $\hat{\textbf{X}}=(\hat{\textbf{R}},\hat{\textbf{p}}^r,\hat{\textbf{f}})$ for ${\textbf{X}}=({\textbf{R}},{\textbf{p}}^r,{\textbf{f}})\in \mathcal U_{\hat{\textbf{X}}}\subset\mathcal M$. In this paper, we call this atlas $\mathcal A^\eta$ the standard atlas for manifold $\mathcal M$.
It is worth noticing that there are countless atlases on $\mathcal M$. For example, based on the standard atlas $\mathcal A^\eta$, we can generate a series of atlases $\mathcal A^\xi=\{\psi_{\hat{\textbf{X}}}\}$ through affine transformations as follows
\begin{equation}
\psi_{\hat{\textbf{X}}}=\textbf{A}_{\hat{\textbf{X}}}\cdot \phi_{\hat{\textbf{X}}},
\end{equation}
where $\textbf{A}_{\hat{\textbf{X}}}$ is an invertible matrix of dimension $(2d+q)\times(2d+q)$. The atlas $\mathcal A^\xi$ is also called an affine atlas in this paper.

\subsection{Extended Kalman Filter on Manifolds}\label{Sec_EKFmanifold}
The considered discrete dynamic system, including the process model (\ref{gen_mot}) and observation model (\ref{gen_obs}), is of the form
\begin{equation}\label{gen_mot}
\textbf{X}_{n}=f(\textbf{X}_{n-1},\textbf{u}_{n-1})\triangleq f_{\textbf{u}_{n-1}}(\textbf{X}_{n-1}),
\end{equation}
\begin{equation}\label{gen_obs}
    \textbf{z}_{n}=h(\textbf{X}_{n}),
\end{equation}
where $\textbf{X}_{n-1},\textbf{X}_{n}\in \mathcal M$ represent the state vectors at the time step $n-1$ and $n$, respectively, $\textbf{u}_{n-1}\in \mathcal W$ is the noise-free control input measurement, and $\textbf{z}_{n}\in \mathcal V$ represents the noise-free observation measurement at the state $\textbf{X}_{n}$. To concisely present our idea, in this section, we simply assume that both the control space $\mathcal W$ and the observation space $\mathcal V$ are Euclidean spaces\footnote{If the control and the observation spaces are general manifolds, we can get the same conclusions of this paper through similar derivations.}. 

\begin{algorithm}[t] 
 \hspace*{0.02in} {$\hat{\textbf{F}}^\epsilon_{n-1}$, $\hat{\textbf{G}}^\epsilon_{n-1}$, and $\hat{\textbf{H}}^\epsilon_{n}$ are defined in (\ref{Jacobians}).}\\[3 bp] 
	 \hspace*{0.01in} {\bf Input:} $\textbf{X}_{n-1|n-1}$, ${\textbf{P}}^\epsilon_{n-1}$, $\hat{\textbf{u}}_{n-1}$, $\hat{\textbf{z}}_{n}$\\
	\hspace*{0.01in} {\bf Output:} $\textbf{X}_{n|n}$, ${\textbf{P}}^\epsilon_{n}$\\
	\hspace*{0.01in} {\bf Propagation:} \\
	\hspace*{0.1in} $ \textbf{X}_{n|n-1}  \leftarrow  {f(\textbf{X}_{n-1|n-1},\hat{\textbf{u}}_{n-1})}$ \\ 
	\hspace*{0.1in} ${\textbf{P}}^\epsilon_{n|n-1} \leftarrow \hat{\textbf{{F}}}^\epsilon_{n-1}{\textbf{P}}^\epsilon_{n-1}  (\hat{\textbf{F}}^\epsilon_{n-1})^{\top} + \hat{\textbf{G}}^\epsilon_{n-1}{\bm{\Sigma}}_{n-1} (\hat{\textbf{G}}^\epsilon_{n-1})^{\top}$\\
	\hspace*{0.01in} {\bf Update:} \\	
	\hspace*{0.1in} $ {\textbf{S}}^\epsilon_{n}\leftarrow \hat{{\textbf{H}}}^\epsilon_{n}{\textbf{P}}^\epsilon_{n|n-1}(\hat{{\textbf{H}}}^\epsilon_{n})^{\top}+{\bm{\Omega}}_{n}$\\  
	\hspace*{0.1in} ${{\textbf{K}}}^\epsilon_{n}\leftarrow {\textbf{P}}^\epsilon_{n|n-1} (\hat{{\textbf{H}}}^\epsilon_{n})^{\top} ({\textbf{S}}^\epsilon_{n})^{-1} $\\ 
	\hspace*{0.1in} ${\textbf{y}}_{n} \leftarrow \hat{\textbf{z}}_{n}-h(\textbf{X}_{n|n-1})$ \\
	\hspace*{0.1in} $\textbf{X}_{n|n} \leftarrow  \rho_{\textbf{X}_{n|n-1}}^{-1}({{\textbf{K}}}^\epsilon_{n}{\textbf{y}}_{n})$\\ 
	\hspace*{0.1in} ${{\textbf{P}}}^\epsilon_{n} \leftarrow ({\textbf{I}}-{{\textbf{K}}}^\epsilon_{n}\hat{{\textbf{H}}}^\epsilon_{n}){\textbf{P}}^\epsilon_{n|n-1}$\\
	\caption{General EKF on a Manifold for a given Atlas}
 \label{EKF_Framework}
\end{algorithm}

However, in practice, only the noisy measurements are available for state estimation. We assume all the noises are ``Gaussian-like". 
Then, for a control input $\textbf{u}_n$ and an observation $\textbf{z}_{n}$ lying in Euclidean spaces, we have
\begin{equation}
    \begin{array}{lll}
&\hat{\textbf{u}}_{n-1}&=\textbf{u}_{n-1}+\textbf{w}_{n-1},\\
&\hat{\textbf{z}}_{n}&=\textbf{z}_{n}+\textbf{v}_{n},
    \end{array}
\end{equation}
where $\hat{\textbf{u}}_{n-1}$ and $\hat{\textbf{z}}_{n}$ are respectively the (noisy) sensor measurements in practical applications. $\textbf{w}_{n-1}\sim N(\textbf{0},\bm{\Sigma}_{n-1})$ and $\textbf{v}_{n}\sim N(\textbf{0},\bm{\Omega}_{n})$ are the corresponding Gaussian noises. 

To apply EKF for estimation, we need to select an atlas $\mathcal A^\epsilon=\{\rho_{\textbf{X}}|\textbf{X}\in \mathcal M\}$ for linearization. Correspondingly, the error vector $\bm{\epsilon}$ of state estimate $\hat{\textbf{X}}$ with respect to the ground-truth $\textbf{X}$ can be defined by 
\begin{equation}\label{B_linear2}
    \bm{\epsilon}\triangleq\rho_{\hat{\textbf{X}}}(\textbf{X})\in \mathbb{R}^m.
\end{equation}
An example is shown by (\ref{example_std}).

Then, based on the given atlas $\mathcal A^\epsilon$, we can implement the EKF algorithm (summarized in Alg. \ref{EKF_Framework}) for the state estimation of system (\ref{gen_mot})(\ref{gen_obs}), where Jacobian matrices based on the atlas $\mathcal A^\epsilon$ are computed by
\begin{equation}\label{Jacobians}
    \begin{array}{rlll}
         \frac{\partial f}{\partial \textbf{X}}|_{(\hat{\textbf{X}},\hat{\textbf{u}})}:& \hat{\textbf{F}}_{n-1}^\epsilon= \frac{\partial \rho_{\textbf{X}_{n|n-1}}(f( \rho^{-1}_{\textbf{X}_{n-1|n-1}}(\bm{\epsilon}),{\textbf{u}}))}{\partial \bm{\epsilon}}|_{(\textbf{0},\hat{\textbf{u}}_{n-1}))}, 
         \\[3bp]
         \frac{\partial f}{\partial \textbf{u}}|_{(\hat{\textbf{X}},\hat{\textbf{u}})}:& \hat{\textbf{G}}_{n-1}^\epsilon= \frac{\partial {\rho_{\textbf{X}_{n|n-1}}}(f( \rho^{-1}_{\textbf{X}_{n-1|n-1}}(\bm{\epsilon}),\textbf{u}))}{\partial \textbf{u}}|_{(\textbf{0},\hat{\textbf{u}}_{n-1})},  \\[3bp]
         \frac{\partial h}{\partial \textbf{X}}|_{\hat{\textbf{X}}}:& \hat{\textbf{H}}_{n}^\epsilon= \frac{\partial h(\rho^{-1}_{\textbf{X}_{n|n-1}}(\bm{\epsilon}))}{\partial \bm{\epsilon}}|_{\textbf{0}}. 
    \end{array}
\end{equation}

\subsection{Observability Analysis}
The observability analysis, which is based on the dimension of unobservable subspace, can be used to identify degenerate situations causing the inconsistency issue \cite{con2,PointLinePlane}. 

We first introduce the concepts about observability of the underlying dynamic systems. Considering the discrete dynamic system described by (\ref{gen_mot}) and (\ref{gen_obs}), the $k$-order observable subspace of this underlying dynamic system can be defined by \cite{ObsSys,ObsSys2,ObsSys3}
\begin{equation}
    \begin{array}{rl}
         \mathcal O_0(\textbf{X}_0)&= \mathop{\text{span}} \{\frac{\partial h}{\partial \textbf{X}}|_{\textbf{X}_0}\}, \\ \\
         \mathcal O_k(\textbf{X}_0)&= \mathop{\text{span}} \{\frac{\partial h(f_{\textbf{u}_{k-1}}\circ \cdots 
 \circ f_{\textbf{u}_0})}{\partial \textbf{X}}|_{\textbf{X}_0}\}\cup \mathcal O_{k-1}(\textbf{X}_0),
    \end{array}
\end{equation}
where $\textbf{X}_0$ is the starting state (the $0$-th step state), and $\{\textbf{u}_i|i=0,1,\cdots\}$ is the sequence of (ground-truth) control inputs.

 Given an atlas $\mathcal A^\epsilon$, $\mathcal O_k(\textbf{X}_0)$ can be  represented by forms of matrix, namely 
 the $k$-order observability matrices, 
 \begin{equation}\label{Ob_Matrix}
\mathcal{{\mathbf{O}}}^\epsilon_{k}(\textbf{X}_0)=\left[
		\begin{array}{cccccc}
			{\textbf{H}}^\epsilon_0\\
		{\textbf{H}}^\epsilon_1{\textbf{F}}^\epsilon_{0,0}\\
			\vdots\\
			{\textbf{H}}_{k}^\epsilon{\textbf{F}}^\epsilon_{k-1,0}
		\end{array}
		\right],
 \end{equation}
 where ${\textbf{F}}^\epsilon_{i,0}={\textbf{F}}^\epsilon_i{\textbf{F}}^\epsilon_{i-1}\cdots {\textbf{F}}^\epsilon_0$ is the transition matrix, ${\textbf{F}}^\epsilon_j$ and ${\textbf{H}}^\epsilon_{i}$ are the Jacobian matrices evaluated at ground-truth values.

The $k$-order unobservable subspace, $\mathcal O_k^\perp$, is the orthogonal complement of observable subspace. Through the given atlas $\mathcal A^\epsilon$, it can be represented by the right nullspace of the $k$-order observability matrix, $\text{null}(\textbf{O}^\epsilon_k)$. We use the symbol $\textbf{N}^\epsilon_k$ to denote $\text{null}(\textbf{O}^\epsilon_k)$,  which is also called the $k$-order unobservable subspace in this paper.

 
The dimensions of $\mathcal O_k$ and $\mathcal O^\perp_k$ can be defined by
 \begin{equation}\label{def_dimO}
 \begin{array}{rl}
     \text{dim}(\mathcal O_k(\textbf{X}_0))\triangleq\text{rank}(\textbf{O}^\epsilon_k(\textbf{X}_0)), \ \forall k\geq 0,\\
     \text{dim}(\mathcal O^\perp_k(\textbf{X}_0))\triangleq\text{dim}(\textbf{N}^\epsilon_k(\textbf{X}_0)), \ \forall k\geq 0.
     \end{array}
 \end{equation} 
These dimensions of subspaces as well as the rank of matrices are the intrinsic properties of the underlying dynamic system, which are independent of the atlases (see Lemma \ref{equRank}). 

\begin{myLemma}\label{equRank}
    Suppose $\mathcal M$ is a manifold, and $\mathcal A^\eta=\{\phi_{\hat{\textbf{X}}}\}$, $\mathcal A^\xi=\{\psi_{\hat{\textbf{X}}}\}$ are two arbitrary atlases on it, then
    \begin{equation}\label{eq_equRank}
        \begin{array}{rl}
     \text{rank}(\textbf{O}^\eta_k)=\text{rank}(\textbf{O}^\xi_k), \ \forall k\geq 0,\\
     \text{dim}(\textbf{N}^\eta_k)=\text{dim}(\textbf{N}^\xi_k), \ \forall k\geq 0,
     \end{array}
    \end{equation}
    where the matrices with superscript of $\eta$ and $\xi$ are computed based on $\mathcal A^\eta=\{\phi_{\hat{\textbf{X}}}\}$ and $\mathcal A^\xi=\{\psi_{\hat{\textbf{X}}}\}$, respectively.
\end{myLemma}
\begin{proof}
    See Appendix \ref{equRank_proof}.
\end{proof}


Next, we will introduce the observability of EKF systems. Based on an atlas $\mathcal A^\epsilon$, we can implement the corresponding EKF algorithm for state estimation (see Sec. \ref{Sec_EKFmanifold}). By computing the Jacobian matrices (\ref{Jacobians}), the $k$-order observability matrix of this EKF system is \cite{con2}
 \begin{equation}\label{EKF_Ob_Matrix}
\hat{\mathcal{{\mathbf{O}}}}^\epsilon_{k}(\textbf{X}_0)=\left[
		\begin{array}{cccccc}
			\hat{\textbf{H}}^\epsilon_0\\
		\hat{\textbf{H}}^\epsilon_1\hat{\textbf{F}}^\epsilon_{0}\\
			\vdots\\
			\hat{\textbf{H}}_{k}^\epsilon\hat{\textbf{F}}^\epsilon_{k-1}\cdots\hat{\textbf{F}}^\epsilon_{0}
		\end{array}
		\right],
 \end{equation}
where all the Jacobian matrices are evaluated at the state estimates. The unobservable subspace of EKF system is denoted by $\hat{\textbf{N}}^\epsilon$ which is the right nullspace of $\hat{\textbf{O}}^\epsilon$\footnote{In the following, the representations of  observability matrix and unobservable subspace of EKF are denoted by the symbol `` $\hat{ }$ ".}.


According to \cite{con4}, in order to achieve consistent estimation, an EKF is required to satisfy the observability constraint described by the following definition.



 \begin{myDef}\label{OC_def}
The EKF algorithm based on an atlas $\mathcal A^\epsilon$ preserves the correct observability property if it has the same unobservable dimension as underlying dynamic system (\ref{gen_mot})(\ref{gen_obs}), i.e. 
\begin{equation}\label{EKF_OC}
    \text{dim}({\hat{\textbf{N}}}_k^\epsilon(\textbf{X}_{0|0}))=\text{dim}({\textbf{N}}_k^\epsilon(\textbf{X}_{0})), 
\end{equation}
$\forall \textbf{X}_{0|0},\textbf{X}_{0}\in \mathcal M,\ \forall k\geq0.$
\end{myDef}



\section{Sufficient and Necessary Conditions for Observability Maintenance}\label{Sec_SNC4OM}

In this section, we start with an example to introduce our findings that the observability discrepancy may be caused by the dependence of unobservable subspace on the state values. Then, from the perspective of the dependence of unobservable subspace on the state values, we theoretically prove some sufficient and necessary conditions for an EKF to preserve the correct observability, confirming our findings.

\subsection{Motivating Example}\label{example4SN}

We hope the EKF model for state estimation could share the same dimension of unobservable subspace as the underlying dynamic system. However, due to the inconsistent linearization points for Jacobians, the unobservable subspace of EKF model may have less dimensions than that of the underlying dynamic system. And the previous studies \cite{con2,con4,con1,con3} have shown that inconsistency issue will occur in such case.

Taking the point feature based SLAM as an example \cite{con4,con3}, the state space $\mathcal M$ containing all the states $\textbf{X}$ with $K$ point features is defined by
\begin{equation}
    \mathcal M=\{\textbf{X}=(\textbf{R},\textbf{p}^r,\textbf{p}^{f_1},\cdots,\textbf{p}^{f_K})\}\cong\mathbb{SO}(d)\times \mathbb{R}^d\times\mathbb{R}^{dK},
\end{equation}
where $d=2,3$ is the dimension of scenarios, $\textbf{R} \in \mathbb{SO}(d)$ and $\textbf{p}^r \in \mathbb{R}^d$ represent the rotation and the position of the robot, respectively, and $\textbf{p}^{f_j} \in \mathbb{R}^d$ is the position of the $j$-th point feature, all described in the global coordinate system\footnote{In the remaining of this paper, without losing generality, we sometimes assume that there is only one feature, i.e. $K=1$, to simplify the equations.}.

The process model (\ref{gen_mot}) is 
\begin{equation}\label{mot_PointSLAM}
\begin{array}{lll}
(\textbf{R}_{n},\textbf{p}^r_{n},\textbf{p}^{f}_{n})=(\textbf{R}_{n-1}{\textbf{R}}^u_{n-1},\textbf{p}^r_{n-1}+\textbf{R}_{n-1}{\textbf{p}}_{n-1}^u,\textbf{p}^{f}_{n-1}),
\end{array}    
\end{equation}
where ${\textbf{U}}_{n-1}=({\textbf{R}}^u_{n-1},{\textbf{p}}^u_{n-1})\in \mathbb{SO}(d)\times \mathbb{R}^d$ is the odometry. 

The observation model (\ref{gen_obs}) is 
\begin{equation}\label{obs_PointSLAM}
\begin{array}{lll}
    \textbf{z}_{n}
    &=\textbf{R}_{n}^{\top}(\textbf{p}^{f}_{n}-\textbf{p}^{r}_{n}).
\end{array}    
\end{equation}
Based on the standard atlas $\mathcal A^\eta=\{\phi_{\hat{\textbf{X}}}\}$, where the chart $\phi_{\hat{\textbf{X}}}$ is shown in (\ref{example_std}), the Std-EKF  ($\mathbb{SO}(d)$-EKF) can be implemented for state estimation. 

We can find that the unobservable subspaces of Std-EKF and underlying dynamic system are respectively
\begin{equation}\label{nullspace_2Dpoint}
    \hat{\textbf{N}}^\eta(\textbf{X}_{0|0})=\mathop{\text{span}} _{col.}\left[
		\begin{array}{cccccc}
			\textbf{0}_{1\times 2}\\
			 \textbf{I}_{2}\\
			 \textbf{I}_{2}
		\end{array}
		\right],\ \textbf{N}^\eta(\textbf{X}_0)=\mathop{\text{span}} _{col.}\left[
		\begin{array}{cccccc}
			1& \textbf{0}_{1\times 2}\\
			\textbf{J}\textbf{p}^r_0& \textbf{I}_{2}\\
			\textbf{J}\textbf{p}^f_0& \textbf{I}_{2}
		\end{array}
		\right],
\end{equation}
if $d=2$, where $\textbf{J}=\left[
	\begin{array}{cccccc}
		0& -1\\ 
		1& 0
	\end{array}
	\right]$;
and 
\begin{equation}\label{nullspace_3Dpoint}
\begin{array}{ll}
    \hat{\textbf{N}}^\eta(\textbf{X}_{0|0})&=\mathop{\text{span}} _{col.}\left[
		\begin{array}{cccccc}
			 \textbf{0}_{3\times3}\\
			 \textbf{I}_{3}\\
			 \textbf{I}_{3}
		\end{array}
		\right],\\
        \textbf{N}^\eta(\textbf{X}_0)&=\mathop{\text{span}} _{col.}\left[
		\begin{array}{cccccc}
			\textbf{I}_3& \textbf{0}_{3\times3}\\
			-(\textbf{p}^r_0)^\land& \textbf{I}_{3}\\
			-(\textbf{p}^f_0)^\land& \textbf{I}_{3}
		\end{array}
		\right],
        \end{array}
\end{equation}
if $d=3$, where $(\cdot)^\land$ is the skew operation.

As revealed and analyzed by previous studies \cite{con2,con4,con1,ic_SLAM1,con3}, the reason for the inconsistency issue in Std-EKFs is the difference in unobservable subspaces (see (\ref{nullspace_2Dpoint}) and (\ref{nullspace_3Dpoint})) between the EKF models and the underlying dynamic systems. 

In contrast, the state-of-the-art method, right invariant EKF (RI-EKF) \cite{RIEKF_2DSLAM,con3}, can naturally preserve the correct observability property and thus improve consistency. The atlas, $\mathcal A^\gamma=\{\pi_{\hat{\textbf{X}}}\}$, applied for RI-EKF is constructed based on a special Lie group $\mathbb{SE}_{K+1}(d)$\footnote{The definition of $\mathbb{SE}_{K+1}(d)$ can be found in \cite{RIEKF_2DSLAM,con3}.} as 
\begin{equation}\label{chart_RI_pointSLAM}
    \pi_{\hat{\textbf{X}}}(\textbf{X})=\text{log}^{\mathbb{SE}_{K+1}(d)}(\textbf{X}\hat{\textbf{X}}^{-1}),
\end{equation}
where $\textbf{X},\hat{\textbf{X}}^{-1}\in \mathbb{SE}_{K+1}(d)$, and the function $\text{log}^{\mathbb{SE}_{K+1}(d)}$ is the inverse of exponential map of Lie group $\mathbb{SE}_{K+1}(d)$. Based on this atlas $\mathcal A^\gamma$, the unobservable subspaces of RI-EKF and underlying dynamic system are respectively
\begin{equation}\label{nullspaceRI_2Dpoint}
    \hat{\textbf{N}}^\gamma(\textbf{X}_{0|0})=\textbf{N}^\gamma(\textbf{X}_0)=\mathop{\text{span}} _{col.}\left[
		\begin{array}{cccccc}
			1& \textbf{0}_{1\times 2}\\
			\textbf{0}_{2\times 1}& \textbf{I}_{2}\\
			\textbf{0}_{2\times 1}& \textbf{I}_{2}
		\end{array}
		\right],
\end{equation}
if $d=2$, and 
\begin{equation}\label{nullspaceRI_3Dpoint}
    \hat{\textbf{N}}^\gamma(\textbf{X}_{0|0})=\textbf{N}^\gamma(\textbf{X}_0)=\mathop{\text{span}} _{col.}\left[
		\begin{array}{cccccc}
			\textbf{I}_3& \textbf{0}_{3\times3}\\
			\textbf{0}_{3\times3}& \textbf{I}_{3}\\
			\textbf{0}_{3\times3}& \textbf{I}_{3}
		\end{array}
		\right],
\end{equation}
if $d=3$.

The EKF, as introduced in Sec. \ref{Sec_EKFmanifold}, in general linearizes at the different values (i.e. $\textbf{X}_{n|n}\neq \textbf{X}_{n|n-1}$) of state $\textbf{X}_n$ for update and the next step prediction. As a result, we can find that the unobservable directions (the first column of $\textbf{N}^\eta$ in (\ref{nullspace_2Dpoint}) and the first three columns of $\textbf{N}^\eta$ in (\ref{nullspace_3Dpoint})) that depend on the values of state $\textbf{X}_0$ are dismissed in unobservable subspaces of Std-EKF, leading to the observability discrepancy. In contrast, all the unobservable directions in (\ref{nullspaceRI_2Dpoint}) and (\ref{nullspaceRI_3Dpoint}) are independent of the values of state $\textbf{X}_0$, and they are preserved in the unobservable subspace of RI-EKF model.
Although this finding has not been rigorously proven in general, it helps some researchers to explain the observability discrepancy between Std-EKF and underlying dynamic system for other specific problems (such as \cite{Zhuqing_IEKFlocalization}). 

\subsection{Sufficient and Necessary Conditions}

As shown in the example in Sec. \ref{example4SN}, it seems that the observability discrepancy can be caused by the dependence of unobservable subspace on the state values. Is this conjecture valid in general? It also raises a further question: will the observability discrepancy be avoided if the unobservable subspace is independent of the state values?


To answer these questions, some sufficient and necessary conditions for observability maintenance are proved. 

Before presenting the theorems, there are a few assumptions for the considered systems in this paper.
\begin{assumption}\label{asmp1}
The system (\ref{gen_mot})(\ref{gen_obs}) is invertible, i.e. $f_{\textbf{u}}$ is invertible for any control input $\textbf{u}$.
\end{assumption}

\begin{assumption}\label{asmp2}
  For each integer $k \geq 0$, the $k$-order observability is independent of $\textbf{X}_0$ and the trajectory, i.e. $$\text{dim}(\mathcal O_k(\textbf{X}_0))=l_k \ \ (\text{constant}),$$
  $\forall \textbf{X}_0\in \mathcal M,\  \forall \{\textbf{u}_i|i=0,1,\cdots\}.$
\end{assumption}

\begin{assumption}\label{asmp3}
For all $k\geq1$, there is
$$\text{dim}(\mathcal O^\perp_k)=\text{dim}(\mathcal O^\perp_1).$$
\end{assumption}



\begin{remark}
In many practical dynamic systems, such as SLAM \cite{Obj1,PointLinePlane}, map based localization \cite{Zhuqing_IEKFlocalization}, localization and tracking \cite{ic_VILTT}, and legged robot \cite{ic_LegRob}, Assumptions \ref{asmp1}-\ref{asmp3} usually hold.
\end{remark}



\begin{theorem}\label{MT1}
Given an atlas $\mathcal A^\epsilon$, and 
assume that one of the following two conditions holds:

(i) $\text{null}(\textbf{H}^\epsilon(\textbf{X}_0))= \bar{\textbf{N}}^\epsilon_0, \forall \textbf{X}_0\in \mathcal M$, where $\bar{\textbf{N}}^\epsilon_0$ is independent \hspace*{0.23in} of the values of state $\textbf{X}_0$;

(ii) $\text{null}(\textbf{H}^\epsilon(\textbf{X}_1)\textbf{F}^\epsilon(\textbf{X}_1,\textbf{X}_0))\subset \text{null}(\textbf{H}^\epsilon(\textbf{X}_0)), \forall \textbf{X}_0,\textbf{X}_1\in \mathcal M.$

\noindent Then, if the corresponding EKF  preserves the correct observability property, 
the unobservable subspaces of this EKF algorithm and underlying dynamic system will be the same constant space, i.e. $$\hat{\textbf{N}}_k^\epsilon(\textbf{X}_{0|0})=\textbf{N}^\epsilon_k(\textbf{X}_0)=\bar{\textbf{N}}^\epsilon, \forall k \geq 1,$$ where $\bar{\textbf{N}}^\epsilon$ is a constant space unrelated to $\textbf{X}_{0|0}$ and $\textbf{X}_0$.

\end{theorem}

\begin{proof}
See Appendix \ref{MT1_proof}.
\end{proof}
\vspace{2mm}

Through the example shown in Sec. \ref{example4SN},
we have found that the dependence of unobservable subspace on the state values may lead to the observability discrepancy. According to Theorem \ref{MT1}, we can affirm this conjecture. For the general systems holding certain conditions, an atlas cannot make its EKF satisfy the observability constraint (\ref{EKF_OC}) if its unobservable subspace of underlying dynamic system are dependent on the state values. Further, the next Theorem \ref{MT2} states that the inverse proposition also holds true.




\begin{theorem} \label{MT2}
    Given an atlas $\mathcal A^\epsilon$, if $\textbf{N}^\epsilon_k(\textbf{X}_0)\equiv\bar{\textbf{N}}^\epsilon, \forall k\geq1,$ is a constant space independent of the values of state $\textbf{X}_0$, then the corresponding EKF based on such atlas satisfies the observability constraint (\ref{EKF_OC}).
\end{theorem}
\begin{proof}
    See Appendix \ref{MT2_proof}.
\end{proof}

Through Theorem \ref{MT1} and Theorem \ref{MT2}, we can directly obtain the following sufficient and necessary condition for observability maintenance of an EKF.

\begin{theorem}\label{MT_SN}
Given an atlas $\mathcal A^\epsilon$, and 
assume that one of the following two conditions holds:

(i) $\text{null}(\textbf{H}^\epsilon(\textbf{X}_0))= \bar{\textbf{N}}^\epsilon_0, \forall \textbf{X}_0\in \mathcal M$, where $\bar{\textbf{N}}^\epsilon_0$ is independent \hspace*{0.23in} of the values of state $\textbf{X}_0$;

(ii) $\text{null}(\textbf{H}^\epsilon(\textbf{X}_1)\textbf{F}^\epsilon(\textbf{X}_1,\textbf{X}_0))\subset \text{null}(\textbf{H}^\epsilon(\textbf{X}_0)), \forall \textbf{X}_0,\textbf{X}_1\in \mathcal M.$

\noindent Then, the corresponding EKF  preserves the correct observability property, 
if and only if $\textbf{N}^\epsilon_k(\textbf{X}_0)\equiv\bar{\textbf{N}}^\epsilon, \forall k\geq1,$ is a constant space unrelated to the values of state $\textbf{X}_0$.

\end{theorem}

Based on the above theorems, we can discover that if we could find an atlas that makes the unobservable subspace independent of the state values, the corresponding EKF will be able to satisfy the observability constraint.
In the next section, we exploit these conditions to propose a method to construct such atlases for the general state estimation problems.

\section{Affine EKF}
In this section, we develop the framework of consistent Aff-EKF which can be regarded as a modification of the inconsistent EKF (considered as Std-EKF) by proper affine transformations. At the beginning, the theoretical foundations about its feasibility are first derived. Then, the procedure to find such proper affine transformations is proposed. Finally, we analyzed the modification of Std-EKF by Aff-EKF from the perspective of covariance manipulation.   


\subsection{Theoretical Foundations}

For a dynamic system (\ref{gen_mot})(\ref{gen_obs}), suppose we have an atlas $\mathcal A^\eta=\{ \phi_{\hat{\textbf{X}}}\}$ (such as (\ref{example_std})) for the state space $\mathcal M$ regarded as the standard atlas that is easy to come up with. However, its corresponding EKF (Std-EKF) is typically unable to preserve the correct observability property, for example, SLAM \cite{con4,Obj1,con3}. 

To derive a consistent EKF, we focus on a series of affine atlases, $$\mathcal A^\xi=\{\psi_{\hat{\textbf{X}}}=\textbf{A}_{\hat{\textbf{X}}}\cdot \phi_{\hat{\textbf{X}}}|\phi_{\hat{\textbf{X}}}\in\mathcal A^\eta\},$$
where $\textbf{A}_{\hat{\textbf{X}}}$ is an invertible matrix. Utilizing the sufficient and necessary conditions for observability maintenance proved in Sec. \ref{Sec_SNC4OM}, we hope to find an affine atlas among them, such that the corresponding EKF satisfies the observability constraint.

At first, the relationship between the representations of the unobservable subspace via standard atlas and affine atlas is shown in Lemma \ref{Lemma1}.
\begin{myLemma} \label{Lemma1}
Suppose $\textbf{N}_k^{{\eta}}(\textbf{X}_0)$ and $\textbf{N}_k^{{\xi}}(\textbf{X}_0)$  are the $k$-order unobservable subspace of ideal system (\ref{gen_mot})(\ref{gen_obs}) represented by the basic atlas $\mathcal A^\eta=\{\phi_\textbf{X}\}$ and the affine atlas $\mathcal A^\xi=\{\textbf{A}_\textbf{X} \phi_\textbf{X}\}$, respectively. Then, the relationship between $\textbf{N}_k^{{\eta}}(\textbf{X}_0)$ and $\textbf{N}_k^{{\xi}}(\textbf{X}_0)$ is given by $$\textbf{N}_k^{{\xi}}(\textbf{X}_0)=\textbf{A}_{\textbf{X}_0}\textbf{N}_k^{{\eta}}(\textbf{X}_0),$$ where $\textbf{X}_0$ is the state at $0$-th step. 
\end{myLemma}

\begin{proof}
    See Appendix \ref{Lemma1_proof}.
\end{proof}
\vspace{2mm}

Then, by Theorem \ref{MT2} and Lemma \ref{Lemma1}, we can directly obtain the following theorem (Theorem \ref{MT3}) which provides a clue for designing an observability preserved Aff-EKF.


\begin{theorem}\label{MT3}
Given an atlas $\mathcal A^\eta=\{\phi_{\textbf{X}}\}$, and assuming that ${\textbf{N}}^{{\eta}}_1(\textbf{X}_0)={\textbf{N}}_k^{{\eta}}(\textbf{X}_0)$, $\forall k\geq 1$, $\forall \textbf{X}_0\in \mathcal M$.
Then, if there is an invertible matrix $\textbf{A}_{\textbf{X}_0}$ for each state $\textbf{X}_0\in \mathcal M$ making $\textbf{A}_{\textbf{X}_0}{\textbf{N}}_1^{{\eta}}(\textbf{X}_0)$ be a constant space for any $\textbf{X}_0\in \mathcal M$, 
the corresponding EKF based on the atlas $\mathcal A^\xi=\{\textbf{A}_\textbf{X}\phi_{\textbf{X}}\}$ will naturally maintain the correct observability property.
\end{theorem}

\subsection{Affine EKF through Affine Transformations}


According to Theorem \ref{MT3}, we can change the Std-EKF through proper affine transformations to maintain the correct observability. The key point is to find a proper affine transformation $\textbf{A}_{\textbf{X}_0}$ for each state $\textbf{X}_0\in \mathcal M$ which eliminates the dependence of unobservable subspace on state. According to the basic property of invertible matrices (Lemma \ref{Lemma_elementM}) \cite[Theorem 1.1]{Intro_LA}, we can obtain such $\textbf{A}_{\textbf{X}_0}$ by the composition of a sequence of elementary row operations. 

\begin{myLemma}\label{Lemma_elementM}
    (\cite[Theorem 1.1]{Intro_LA}) The matrix $\textbf{A}$ is invertible if and only if it can be expanded into a product of elementary matrices , i.e.
$$\textbf{A}=\textbf{E}_a\textbf{E}_{a-1}\cdots\textbf{E}_1,$$
where $\textbf{E}_i$ is an elementary matrix left multiplication by which represents a row switching, row multiplication, or row addition transformation \cite{Intro_LA}.
\end{myLemma}

Let's take 3D point feature based SLAM as an example. The unobservable subspace and its basis of underlying dynamic system under the standard atlas have been shown in (\ref{nullspace_3Dpoint}). We notice that the second and third block elements in the first block column are related to the state values. To eliminate these state values, we can respectively add to the second and third block row a multiple of the first block row as follows,
\begin{equation}\label{3Dpoint_eleOP1}
    \left[
		\begin{array}{cccccc}
			\textbf{I}_3& \textbf{0}_{3\times3}\\
			-(\textbf{p}^r)^\land& \textbf{I}_{3}\\
			-(\textbf{p}^f)^\land& \textbf{I}_{3}
		\end{array}
		\right]\underrightarrow{\text{L}_2+(\textbf{p}^r)^\land\cdot\text{L}_1} 
  \left[
		\begin{array}{cccccc}
			\textbf{I}_3& \textbf{0}_{3\times3}\\
			\textbf{0}_{3\times3}& \textbf{I}_{3}\\
			-(\textbf{p}^f)^\land& \textbf{I}_{3}
		\end{array}
		\right],
\end{equation}
\begin{equation}\label{3Dpoint_eleOP2}
    \left[
		\begin{array}{cccccc}
			\textbf{I}_3& \textbf{0}_{3\times3}\\
			\textbf{0}_{3\times3}& \textbf{I}_{3}\\
			-(\textbf{p}^f)^\land& \textbf{I}_{3}
		\end{array}
		\right]\underrightarrow{\text{L}_3+(\textbf{p}^f)^\land\cdot\text{L}_1} 
  \left[
		\begin{array}{cccccc}
			\textbf{I}_3& \textbf{0}_{3\times3}\\
			\textbf{0}_{3\times3}& \textbf{I}_{3}\\
			\textbf{0}_{3\times3}& \textbf{I}_{3}
		\end{array}
		\right], \hspace{9mm}
\end{equation}
where the $i$-th block row is labeled as $\textbf{L}_i$. The corresponding elementary matrices of (\ref{3Dpoint_eleOP1}) and (\ref{3Dpoint_eleOP2}) are, respectively,
\begin{equation}
 \begin{array}{rl}
    \textbf{E}_{\textbf{X},1}=\left[
		\begin{array}{cccc}
			\textbf{I}_3& \textbf{0}_{3\times3}& \textbf{0}_{3\times3}\\
			(\textbf{p}^r)^\land& \textbf{I}_3& \textbf{0}_{3\times3}\\
			\textbf{0}_{3\times3}&\textbf{0}_{3\times3}& \textbf{I}_3
		\end{array}
		\right],\\
      \textbf{E}_{\textbf{X},2}=\left[
		\begin{array}{cccc}
			\textbf{I}_3& \textbf{0}_{3\times3}& \textbf{0}_{3\times3}\\
			\textbf{0}_{3\times3}& \textbf{I}_3& \textbf{0}_{3\times3}\\
(\textbf{p}^f)^\land&\textbf{0}_{3\times3}& \textbf{I}_3
		\end{array}
		\right].
 \end{array}
 \end{equation}
 Therefore, the proper affine transformations can be obtained by
 \begin{equation}
\textbf{A}_{\textbf{X},1}=\textbf{E}_{\textbf{X},2}\textbf{E}_{\textbf{X},1}=\left[
		\begin{array}{cccc}
			\textbf{I}_3& \textbf{0}_{3\times3}& \textbf{0}_{3\times3}\\
			(\textbf{p}^r)^\land& \textbf{I}_3& \textbf{0}_{3\times3}\\
		(\textbf{p}^f)^\land&\textbf{0}_{3\times3}& \textbf{I}_3
		\end{array}
		\right].
 \end{equation}
 According to Theorem \ref{MT3}, the EKF based on the atlas, 
\begin{equation}\label{3Dpoint_AffAtlas1}
    \mathcal A^\xi_1=\{\psi_{\hat{\textbf{X}},1}=\textbf{A}_{\hat{\textbf{X}},1}\cdot \phi_{\hat{\textbf{X}}}| \phi_{\hat{\textbf{X}}}\in\mathcal A^\eta\},
\end{equation}
preserves the correct observability property for 3D point feature based SLAM.

We can find that the proper affine matrix $\textbf{A}_\textbf{X}$ is not unique for each state. If we perform the eliminations from another basis of (\ref{nullspace_3Dpoint}), such as 
\begin{equation}
\mathop{\text{span}} _{col.}\left[
		\begin{array}{cccccc}
			\textbf{I}_{3}& \textbf{0}_{3\times3}\\
			-(\textbf{p}^r)^\land & \textbf{R}\\
			-(\textbf{p}^f)^\land & \textbf{R}
		\end{array}
		\right](=\mathop{\text{span}} _{col.}\left[
		\begin{array}{cccccc}
			\textbf{I}_3& \textbf{0}_{3\times3}\\
			-(\textbf{p}^r)^\land& \textbf{I}_{3}\\
			-(\textbf{p}^f)^\land& \textbf{I}_{3}
		\end{array}
		\right]),
\end{equation}
the state values can be eliminated by
\begin{equation}
    \left[
		\begin{array}{cccccc}
			\textbf{I}_{3}& \textbf{0}_{3\times3}\\
			-(\textbf{p}^r)^\land & \textbf{R}\\
			-(\textbf{p}^f)^\land & \textbf{R}
		\end{array}
		\right]\underrightarrow{\substack{ \text{L}_2+(\textbf{p}^r)^\land\cdot\text{L}_1\\\text{L}_3+(\textbf{p}^f)^\land\cdot\text{L}_1\\\textbf{R}^\top\cdot\text{L}_2\\\textbf{R}^\top\cdot\text{L}_3\\}} 
  \left[
		\begin{array}{cccccc}
			\textbf{I}_3& \textbf{0}_{3\times3}\\
			\textbf{0}_{3\times3}& \textbf{I}_{3}\\
			\textbf{0}_{3\times3}& \textbf{I}_{3}
		\end{array}
		\right].
\end{equation}
Then, we obtain another proper affine atlas, 
\begin{equation}\label{3Dpoint_AffAtlas2}
    \mathcal A^\xi_2=\{\psi_{\hat{\textbf{X}},2}=\textbf{A}_{\hat{\textbf{X}},2}\cdot \phi_{\hat{\textbf{X}}}| \phi_{\hat{\textbf{X}}}\in\mathcal A^\eta\},
\end{equation}
where
$$\textbf{A}_{\textbf{X},2}=\left[
		\begin{array}{cccc}
			\textbf{I}_3& \textbf{0}_{3\times3}& \textbf{0}_{3\times3}\\
			\textbf{R}^\top(\textbf{p}^r)^\land& \textbf{R}^\top& \textbf{0}_{3\times3}\\
\textbf{R}^\top(\textbf{p}^f)^\land&\textbf{0}_{3\times3}& \textbf{R}^\top
		\end{array}
		\right].$$
We can easily check that both EKFs implemented by these two atlas (\ref{3Dpoint_AffAtlas1}) and (\ref{3Dpoint_AffAtlas2}) preserve the correct observability property, confirming our Theorem \ref{MT3}.

As shown in the above example, similar to Gaussian elimination method, the proposed procedure uses a sequence of elementary row operations to change the basis of unobservable subspace until it is independent of state values. The general procedure to find proper affine transformations is summarized in Alg. \ref{Alg3:FindA}. 


\begin{algorithm}[t]
	\hspace*{0.02in} {\bf Input:} Observability matrix $\textbf{O}^\eta(\textbf{X})$ of underlying system\\
	\hspace*{0.02in} {\bf Output:} Matrix $\textbf{A}_\textbf{X}$ for observability preserved Aff-EKF\\
	\hspace*{0.02in} {\bf Process:} \\	
\hspace*{0.1in}  1. $\textbf{N}^\eta(\textbf{X})\leftarrow \text{null}(\textbf{O}(\textbf{X}))$\\
\hspace*{0.1in}  2. Use elementary row operations to transform $\textbf{N}^\eta(\textbf{X})$ into \hspace*{0.25in} a  constant space\\
\hspace*{0.1in}  3.  Find corresponding elementary matrices  \\ \hspace*{0.25in} $\{\textbf{E}_i|i=1,\cdots,a\}$ for the elementary row operations \\
\hspace*{0.1in}  4.  $\textbf{A}_\textbf{X} \leftarrow \textbf{E}_a\textbf{E}_{a-1}\cdots\textbf{E}_1$
 \caption{Procedure for Finding $\textbf{A}_\textbf{X}$}
	\label{Alg3:FindA}
\end{algorithm}

\subsection{Affine EKF through Manipulating Covariance}

\begin{figure}[t]
	\centering
	\includegraphics[width=.46\textwidth]{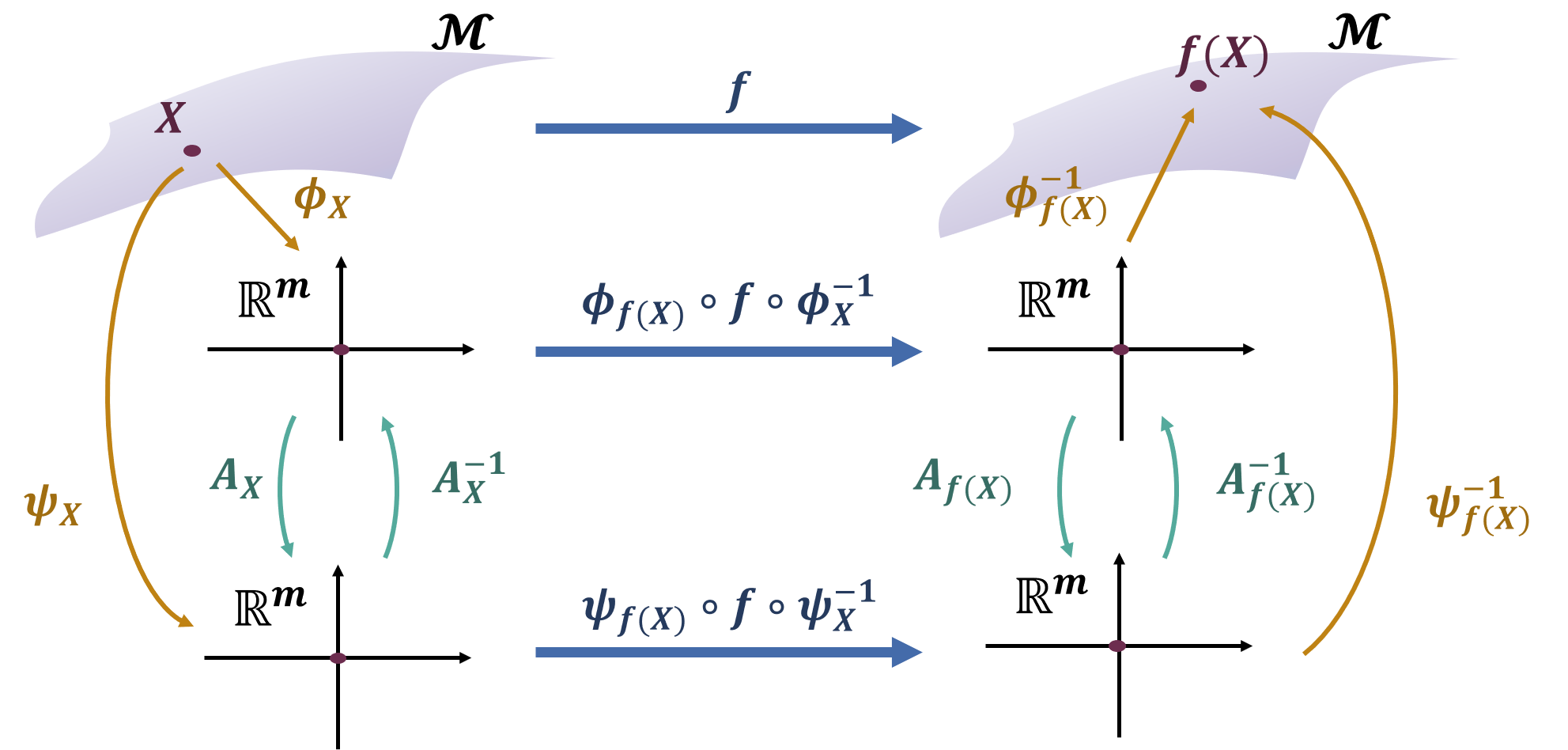}
	\caption{Representations of a Mapping on Manifold.}
	\label{AffRep}
\end{figure}

Suppose there are the standard atlas $\mathcal A^\eta=\{\phi_{\textbf{X}}\}$ and an affine atlas $\mathcal A^\xi=\{\psi_{\textbf{X}}=\textbf{A}_\textbf{X}\cdot\phi_{\textbf{X}}\}$. Their corresponding EKF are denoted as Std-EKF and Aff-EKF, respectively. Fig. \ref{AffRep} shows the relationship between different representations of a mapping by standard atlas and affine atlas \cite{Lee2012}. 

By (\ref{B_linear2}), the error states for Std-EKF and Aff-EKF are, respectively,
\begin{equation}
     \bm{\eta}=\phi_{\hat{\textbf{X}}}(\textbf{X}),\   \bm{\xi}=\psi_{\hat{\textbf{X}}}(\textbf{X}).
\end{equation}
The transformation between them are
\begin{equation}\label{aff_linear2}
\bm{\xi}=\textbf{A}_{\hat{\textbf{X}}}\bm{\eta},
\end{equation}
where \begin{equation}\label{A2}
\textbf{A}_{\hat{\textbf{X}}}=\frac{\partial \psi_{\hat{\textbf{X}}}\circ \phi_{\hat{\textbf{X}}}^{-1}}{\partial \bm{\eta}},
\end{equation}
is also the Jacobian of coordinate transformation from $\phi_{\hat{\textbf{X}}}$ to $\psi_{\hat{\textbf{X}}}$.

Then, by  (\ref{aff_linear2}), when describing the same uncertainty of $\hat{\textbf{X}}$, the equivalent covariance relationship between the representation in Std-EKF and the representation in Aff-EKF is given as follows,
\begin{equation}\label{equ_Cov}
    \textbf{P}^\xi=\textbf{A}_{\hat{\textbf{X}}}\textbf{P}^\eta(\textbf{A}_{\hat{\textbf{X}}})^{\top},
\end{equation}
where $\textbf{P}^\eta$ and $\textbf{P}^{\xi}$ are the covariance matrices represented by the standard atlas $\mathcal A^\eta$ and the affine atlas $\mathcal A^\xi$, respectively.

Further, for the Jacobians in the EKF algorithms, the transformations between the different representations on the standard atlas and the affine atlas are obtained by
\begin{equation}\label{Jaco_trans}
    \begin{array}{llll}
\hat{\textbf{F}}_{n-1}^{\xi}&=\textbf{A}_{\textbf{X}_{n|n-1}}\hat{\textbf{F}}_{n-1}^{\eta}\textbf{A}_{\textbf{X}_{n-1|n-1}}^{-1},\\
\hat{\textbf{G}}_{n-1}^{\xi}&=\textbf{A}_{\textbf{X}_{n|n-1}}\hat{\textbf{G}}_{n-1}^{\eta},\\
\hat{\textbf{H}}_{n}^{\xi}&=\hat{\textbf{H}}_{n}^{\eta}\textbf{A}_{\textbf{X}_{n|n-1}}^{-1}.\\
    \end{array}
\end{equation}

Next, we will explain how Aff-EKF revamps Std-EKF from a perspective of covariance update. Now considering one step EKF process. If we input the same measurements $\hat{\textbf{u}}_{n-1}$ and $\hat{\textbf{z}}_{n}$, the same initial state $\textbf{X}_{n-1|n-1}$ and its equivalent covariance $$\textbf{P}^\xi_{n-1}=\textbf{A}_{\textbf{X}_{n-1|n-1}}\textbf{P}^\eta_{n-1}(\textbf{A}_{\textbf{X}_{n-1|n-1}})^{\top}$$
into Std-EKF and Aff-EKF for one step estimation,
we can easily deduce that: (a) the updated states of this two EKFs are the same, i.e. $\textbf{X}^\xi_{n|n}=\textbf{X}^\eta_{n|n}$; and (b) the updated covariance matrices $\textbf{P}_{n}^{\eta}$ and $\textbf{P}_{n}^{\xi}$, which respectively describe the estimated uncertainty of $\textbf{X}^\eta_{n|n}$ and $\textbf{X}^\xi_{n|n}$, have the following relation
\begin{equation}
    \textbf{P}_{n}^{\xi}=\textbf{A}_{\textbf{X}_{n|n-1}}\textbf{P}_{n}^{\eta}(\textbf{A}_{\textbf{X}_{n|n-1}})^{\top}.
\end{equation}
However, the updated covariances $\textbf{P}_{n}^{\eta}$ and $\textbf{P}_{n}^{\xi}$ are not equivalent as shown in (\ref{equ_Cov}). This is because they are describing the uncertainties of state $\textbf{X}_{n|n}$ instead of $\textbf{X}_{n|n-1}$. Actually, if they are represented by a same atlas (such as standard atlas), we have \begin{equation}\label{alter_affEKF}
    \begin{array}{rl}
         \widetilde{\textbf{P}}_{n}^{\eta}&=\textbf{A}_{\textbf{X}_{n|n}}^{-1}\textbf{P}_{n}^{\xi}\textbf{A}_{\textbf{X}_{n|n}}^{-\top}  \\
         &=(\textbf{A}_{\textbf{X}_{n|n}}^{-1}\textbf{A}_{\textbf{X}_{n|n-1}}) \textbf{P}_{n}^{\eta} (\textbf{A}_{\textbf{X}_{n|n}}^{-1}\textbf{A}_{\textbf{X}_{n|n-1}})^{\top}.
    \end{array}
\end{equation}

From an intuitive viewpoint, the phenomenon of inconsistent issue is that the estimated covariance does not match the uncertainty of state estimate.
The Aff-EKF adjusts this covariance to be consistent with the uncertainty of state estimate through (\ref{alter_affEKF}). Therefore, without considering computational complexity, there is an alternative implementation of Aff-EKF, which is summarized in Alg. \ref{Alg2:Aff-EKF}.
It directly rectifies covariance matrix calculation of Std-EKF, and only requires changing a few lines of code from Std-EKF (the two lines in \textit{Covariance Modification} in Alg. \ref{Alg2:Aff-EKF}).

\begin{algorithm}[t]
	\hspace*{0.02in} {\bf Input:} $\textbf{X}_{n-1|n-1}$, $\widetilde{\textbf{P}}_{n-1}^{\eta}$, $\hat{\textbf{u}}_{n-1}$, $\hat{\textbf{z}}_{n}$\\
	\hspace*{0.02in} {\bf Output:} $\textbf{X}_{n|n}$, $\widetilde{\textbf{P}}^{\eta}_{n}$\\
	\hspace*{0.02in} {\bf Process:} \\
 \hspace*{0.1in} {1. Implementation of Std-EKF:}\\
	\hspace*{0.1in} $ \textbf{X}_{n|n},\ {{\textbf{P}}}_{n}^{\eta}  \leftarrow \textit{Std-EKF}(\textbf{X}_{n-1|n-1},{\textbf{P}}_{n-1}^{\eta},\hat{\textbf{u}}_{n-1},\hat{\textbf{z}}_{n})$\\
 \hspace*{0.1in} {2. Covariance Modification:}\\
\hspace*{0.1in} {$\textbf{L}_{n} \leftarrow \textbf{A}_{\textbf{X}_{n|n}}^{-1}\textbf{A}_{\textbf{X}_{n|n-1}}$\\
\hspace*{0.1in} $\widetilde{\textbf{P}}_{n}^{\eta} \leftarrow \textbf{L}_{n} \textbf{P}_{n}^{\eta} (\textbf{L}_{n})^{\top}$}
	\caption{The Alternative of Aff-EKF Algorithm}
	\label{Alg2:Aff-EKF}
\end{algorithm}

\section{Applications}\label{Sec_Applications}
We showcase our Aff-EKF method in three SLAM applications with different types of features: typical point features (Sec. \ref{Sec_PointSLAM}), point features on a horizontal plane (Sec. \ref{Sec_ConPointSLAM}), and plane features (Sec. \ref{Sec_PlaneSLAM}). For each application, the following subsections will sequentially present the corresponding problem formulations, observability analysis, implementation of Aff-EKF and computational complexity.

\subsection{Point Feature based SLAM}\label{Sec_PointSLAM}

We have taken typical 3D point feature based SLAM as an example in the previous sections. For such problem, the process model (\ref{gen_mot}) and the observation model (\ref{gen_obs}) are given by (\ref{mot_PointSLAM}) and (\ref{obs_PointSLAM}), respectively.


The standard atlas $\mathcal A^\eta$ can be obtained by (\ref{example_std}). As shown in (\ref{nullspace_2Dpoint}) and (\ref{nullspace_3Dpoint}), the corresponding Std-EKF are unable to preserve the correct observability property, leading to inconsistent issue. The observability preserved Aff-EKFs can be obtained by our proposed Alg. \ref{Alg3:FindA}. Derived from different bases of unobservable subspace, two Aff-EKFs, Aff-EKF v1 and Aff-EKF v2, have been given by (\ref{3Dpoint_AffAtlas1}) and (\ref{3Dpoint_AffAtlas2}) for this case, respectively. 

\textbf{Computational Complexity.} Suppose there are $K$ features in the state. Due to the different methods of implementation, the computational complexities of an Aff-EKF using Alg. \ref{EKF_Framework} and Alg. \ref{Alg2:Aff-EKF} may be different by stages.

For the propagation, the cost of our Aff-EKF v1 (Alg. \ref{EKF_Framework}) is $O(K^2)$, the cost of our Aff-EKF v2 (Alg. \ref{EKF_Framework}) is $O(K^3)$, and the costs of our Aff-EKF v1 (Alg. \ref{Alg2:Aff-EKF}) and our Aff-EKF v2 (Alg. \ref{Alg2:Aff-EKF}) are $O(K)$. For the update, the costs of these two different implementations of our Aff-EKF v1 and Aff-EKF v2 are $O(K^2)$. Besides, there are extra $O(K^2)$ and $O(K^3)$ costs for the modifications on covariance in Aff-EKF v1 (Alg. \ref{Alg2:Aff-EKF})\footnote{Here we utilize the unit matrices on the diagonal of the partitioned lower triangular matrix to reduce complexity.} and Aff-EKF v2 (Alg. \ref{Alg2:Aff-EKF}), respectively.

As a comparison, the costs of Std-EKF, FEJ-EKF, OC-EKF, DRI-EKF are $O(K)$ for the propagation and $O(K^2)$ for the update. The costs of RI-EKF are $O(K^2)$ for the propagation and $O(K^2)$ for the update. In addition, there is an extra $O(K^3)$ cost for computing the evaluation points in OC-EKF.

\subsection{SLAM with Plane Constraints on Point Features}\label{Sec_ConPointSLAM}

In this part, the 3D point features are considered to be on a horizontal plane, such as the ground or ceiling. 
Specifically, the z-coordinate values of these point features are the same. To the best of our knowledge, there is no consistent filter, especially RI-EKF, for the SLAM problem with such features. It is not easy to find RI-EKF for the considered problems, and it may not even exist. However, by our proposed Aff-EKF framework, we can easily derive a consistent EKF which naturally preserves the correct observability.

In the considered SLAM problem, the 3D point features, $\textbf{p}^f\in \mathbb{R}^3$, are supposed to be on the ground/ceiling, i.e.
\begin{equation} \label{feature_ground}
    \textbf{p}^f=(p^f_x,p^f_y,c)^\top,
\end{equation}
where the height $c$ is shared for all the features.
The dynamic systems can be given by (\ref{mot_PointSLAM})(\ref{obs_PointSLAM}), which is the same as the typical point feature cases.

\subsubsection{Scenario 1. (The Height c is Known)} 
If the horizontal plane is known in advance, the z-coordinate value of point features, $c$, will be given and thus does not need to be estimated. Hence, the states with $K$ 3D point features on the horizontal plane can be defined by
\begin{equation}
    \textbf{X}=(\textbf{R},\textbf{p}^r,{p}^{f_1}_x,{p}^{f_1}_y,\cdots,{p}^{f_K}_x,{p}^{f_K}_y)\in\mathbb{SO}(3)\times \mathbb{R}^3\times\mathbb{R}^{2K},
\end{equation}
where $\textbf{R}$ and $\textbf{p}^r$ represent the rotation and the position of the robot, respectively, and $({p}^{f_j}_x,{p}^{f_j}_y) \in \mathbb{R}^2$ are the x-coordinate and y-coordinate values of the $j$-th feature, all described in the global coordinate system. 

The standard atlas $\mathcal A^\eta=\{\phi_{\hat{\textbf{X}}}\}$ can be designed by
\begin{equation}
    \phi_{\hat{\textbf{X}}}(\textbf{X})\triangleq(\log^{\mathbb{SO}(3)}(\textbf{R}\hat{\textbf{R}}^\top),{\textbf{p}}^r-\hat{\textbf{p}}^r,{{p}}_x^{f}-\hat{{p}}_x^{f},{{p}}_y^{f}-\hat{{p}}_y^{f})^\top.
\end{equation}

Accordingly, the Jacobian matrices of by the atlas $\mathcal A^\eta$ are shown below.
\begin{equation}
\begin{array}{llll}
     {\textbf{F}}_{n-1}^{\eta}=\left[
		\begin{array}{cccc}
			\textbf{I}_3& \textbf{0}_{3\times 3}& \textbf{0}_{3\times 2}\\
		-(\textbf{p}^r_{n}-\textbf{p}^r_{n-1})^{\land}&  \textbf{I}_3& \textbf{0}_{3\times 2}\\
			\textbf{0}_{2\times 3}&  \textbf{0}_{2\times 3}& \textbf{I}_2
		\end{array}
		\right],\vspace{1mm}\\
     {\textbf{G}}_{n-1}^{\eta}=\left[
		\begin{array}{cccc}
			\textbf{R}_{n-1}&  \textbf{0}_{3\times 3}\\
			\textbf{0}_{3\times 3}&  \textbf{R}_{n-1}\\
			\textbf{0}_{2\times 3}&  \textbf{0}_{2\times 3}
		\end{array}
		\right],\\
     {\textbf{H}}_{n}^{\eta}=\textbf{R}_{n}^{{\top}}\left[\begin{array}{cccc}
     (\textbf{p}^f_{n}-\textbf{p}^r_{n})^{\land} & -\textbf{I}_3 & \left[
		\begin{array}{cccc}
			\textbf{I}_2 \\ \textbf{0}_{1\times 2}
		\end{array}
		\right]\end{array}
		\right].\\
     \end{array}
\end{equation}

Based on the standard atlas $\mathcal A^\eta$, the unobservable subspace of the underlying dynamic system at the state $\textbf{X}$ is obtained as
\begin{equation}
		{\textbf{N}}^{\eta}(\textbf{X})=\mathop{\text{span}} _{col.}\left[
		\begin{array}{cccccc}
			\textbf{0}_{2\times 2}& 
			\textbf{0}_{2\times 1} \\ \textbf{0}_{1\times 2}&1\\
			\textbf{I}_2 
			& \left[
		\begin{array}{cccc}
			-\textbf{p}^r_{y}\\ \textbf{p}^r_{x}
		\end{array}
		\right]\\\textbf{0}_{1\times 2}&0\\
			\textbf{I}_{2}& \left[
		\begin{array}{cccc}
			-\textbf{p}^f_{y}\\ \textbf{p}^f_{x}
		\end{array}
		\right]
		\end{array}
		\right],
	\end{equation}
	where $\textbf{p}^r=({p}^r_{x},{p}^r_{y},{p}^r_{z})$ and $\textbf{p}^f=({p}^f_{x},{p}^f_{y},c)$ are respectively the (true) positions of robot pose and features on the plane. 
We can find that this unobservable subspace $\textbf{N}^\eta$ represented by the standard atlas is related to the state values. Therefore, according to Theorem \ref{MT1}, Std-EKF is unable to preserve the correct observability property.

Through the procedure (Alg. \ref{Alg3:FindA}),
\begin{equation}
    \left[
		\begin{array}{cccccc}
			\textbf{0}_{2\times 2}& 
			\textbf{0}_{2\times 1} \\ \textbf{0}_{1\times 2}&1\\
			\textbf{I}_2 
			& \left[
		\begin{array}{cccc}
			-\textbf{p}^r_{y}\\ \textbf{p}^r_{x}
		\end{array}
		\right]\\\textbf{0}_{1\times 2}&0\\
			\textbf{I}_{2}& \left[
		\begin{array}{cccc}
			-\textbf{p}^f_{y}\\ \textbf{p}^f_{x}
		\end{array}
		\right]
		\end{array}
		\right]\underrightarrow{\substack{\text{L}_3-\left[
		\begin{array}{cccc}
			-\textbf{p}^r_{y}\\ \textbf{p}^r_{x}
		\end{array}
		\right]\cdot\text{L}_2\\\text{L}_5-\left[
		\begin{array}{cccc}
			-\textbf{p}^f_{y}\\ \textbf{p}^f_{x}
		\end{array}
		\right]\cdot\text{L}_2}}\left[
		\begin{array}{cccccc}
			\textbf{0}_{2\times 2}& 
			\textbf{0}_{2\times 1} \\ \textbf{0}_{1\times 2}&1\\
			\textbf{I}_2 
			& \textbf{0}_{2\times 1}\\\textbf{0}_{1\times 2}&0\\
			\textbf{I}_{2}& \textbf{0}_{2\times 1}
		\end{array}
		\right], 
\end{equation}
we can obtain an affine atlas,
\begin{equation}\label{S1_conpoint_AffAtlas}
    \mathcal A^\xi=\{\psi_{\hat{\textbf{X}}}=\textbf{A}_{\hat{\textbf{X}}}\cdot \phi_{\hat{\textbf{X}}}| \phi_{\hat{\textbf{X}}}\in\mathcal A^\eta\},
\end{equation}
where
$$\textbf{A}_{\textbf{X}}=\left[
		\begin{array}{ccccccc}
			\textbf{I}_2& \textbf{0}_{2\times 1}& \textbf{0}_{2\times 2}&\textbf{0}_{2\times 1}&\textbf{0}_{2\times 2}\\
   \textbf{0}_{1\times 2}& 1& \textbf{0}_{1\times 2}&0&\textbf{0}_{1\times 2}\\
   \textbf{0}_{2\times 2}&
			\left[
		\begin{array}{ccccccc}
			\textbf{p}^r_{y}\\ -\textbf{p}^r_{x}
		\end{array}
		\right]& \textbf{I}_2& \textbf{0}_{2\times 1}& \textbf{0}_{2\times 2}\\
  \textbf{0}_{1\times 2}& 0& \textbf{0}_{1\times 2}&1&\textbf{0}_{1\times 2}\\
			\textbf{0}_{2\times 2}&
			\left[
		\begin{array}{cccccccc}
			\textbf{p}^f_{y}\\ -\textbf{p}^f_{x}
		\end{array}
		\right]& \textbf{0}_{2\times 2}&\textbf{0}_{2\times 1}&\textbf{I}_2
		\end{array}
		\right].$$
Then, by Theorem \ref{MT3}, the corresponding Aff-EKF of (\ref{S1_conpoint_AffAtlas}) naturally preserves the correct observability property for the considered problem.

\subsubsection{Scenario 2. (The Height c is Unknown)}
If the horizontal plane is unknown, the z-coordinate value of point features, $c$, is required to be estimated. Hence, the states with $K$ 3D point features on the horizontal plane can be defined by
\begin{equation}
    \textbf{X}=(\textbf{R},\textbf{p}^r, c, {p}^{f_1}_x,{p}^{f_1}_y\cdots, {p}^{f_K}_x,{p}^{f_K}_y)\in\mathbb{SO}(3)\times \mathbb{R}^3\times\mathbb{R}^{1+2K}.
\end{equation}
The standard atlas $\mathcal A^\eta=\{\phi_{\hat{\textbf{X}}}\}$ can be designed by
\begin{equation}
    \phi_{\hat{\textbf{X}}}(\textbf{X})\triangleq(\log^{\mathbb{SO}(3)}(\textbf{R}\hat{\textbf{R}}^\top),{\textbf{p}}^r-\hat{\textbf{p}}^r, c-\hat{c}, {{p}}_x^{f}-\hat{{p}}_x^{f},{{p}}_y^{f}-\hat{{p}}_y^{f})^\top.
\end{equation}

Accordingly, the Jacobian matrices by the atlas $\mathcal A^\eta$ are shown below.
\begin{equation}
\begin{array}{llll}
     {\textbf{F}}^{\eta}_{n-1}=\left[
		\begin{array}{cccc}
			\textbf{I}_3& \textbf{0}_{3\times 3}&\textbf{0}_{3\times 1}& \textbf{0}_{3\times 2}\\
		-(\textbf{p}^r_{n}-\textbf{p}^r_{n-1})^{\land}&  \textbf{I}_3& \textbf{0}_{3\times 1}& \textbf{0}_{3\times 2}\\
  \textbf{0}_{1\times 3}&\textbf{0}_{1\times 3}&1&\textbf{0}_{1\times 2}\\
			\textbf{0}_{2\times 3}&  \textbf{0}_{2\times 3}&\textbf{0}_{2\times 1}& \textbf{I}_2
		\end{array}
		\right],\vspace{1mm}\\ 
     {\textbf{G}}^{\eta}_{n-1}=\left[
		\begin{array}{cccc}
			\textbf{R}_{n-1}&  \textbf{0}_{3\times 3}\\
			\textbf{0}_{3\times 3}&  \textbf{R}_{n-1}\\
   \textbf{0}_{1\times 3}&\textbf{0}_{1\times 3}\\
			\textbf{0}_{2\times 3}&  \textbf{0}_{2\times 3}
		\end{array}
		\right],\vspace{1mm}\\
     {\textbf{H}}^{\eta}_{n}=\textbf{R}_{n}^{{\top}}
     \left[\begin{array}{cccc}
     (\textbf{p}_{n}^f-\textbf{p}_{n}^r)^{\land} & -\textbf{I}_3 & \left[
		\begin{array}{cccc}
			\textbf{0}_{2\times 1} \\ 1
		\end{array}
		\right]&\left[
		\begin{array}{cccc}
			\textbf{I}_2 \\ \textbf{0}_{1\times 2}
		\end{array}
		\right]\end{array}
		\right].\\
     \end{array}
\end{equation}
Based on the standard atlas $\mathcal A^\eta$, the unobservable subspace for the underlying dynamic system at the state $\textbf{X}$ is obtained as
\begin{equation}
		{\textbf{N}}^{\eta}(\textbf{X})=\mathop{\text{span}} _{col.}\left[
		\begin{array}{cccccc}
			\textbf{0}_{2\times 3}& 
			\textbf{0}_{2\times 1} \\ \textbf{0}_{1\times 3}&1\\
			\textbf{I}_3 
			& \left[
		\begin{array}{cccc}
			-\textbf{p}^r_{y}\\ \textbf{p}^r_{x}\\0
		\end{array}
		\right]\\ \left[\begin{array}{cccccc}\textbf{0}_{1\times 2} &1\end{array}\right]&0\\
			\left[\begin{array}{cccccc}\textbf{I}_{2} &\textbf{0}_{2\times 1}\end{array}\right]& \left[
		\begin{array}{cccc}
			-\textbf{p}^f_{y}\\ \textbf{p}^f_{x}
		\end{array}
		\right]
		\end{array}
		\right].
	\end{equation}
Similarly, this representation of unobservable subspace $\textbf{N}^\eta$ by the standard atlas is also related to the state values. Therefore, according to Theorem \ref{MT1}, the Std-EKF is unable to satisfy the observability constraint.
Through the procedure (Alg. \ref{Alg3:FindA}), 
\begin{equation}
\begin{array}{cc}
    \left[
		\begin{array}{cccccc}
			\textbf{0}_{2\times 3}& 
			\textbf{0}_{2\times 1} \\ \textbf{0}_{1\times 3}&1\\
			\textbf{I}_3 
			& \left[
		\begin{array}{cccc}
			-\textbf{p}^r_{y}\\ \textbf{p}^r_{x}\\0
		\end{array}
		\right]\\ \left[\begin{array}{cccccc}\textbf{0}_{1\times 2} &1\end{array}\right]&0\\
			\left[\begin{array}{cccccc}\textbf{I}_{2} &\textbf{0}_{2\times 1}\end{array}\right]& \left[
		\begin{array}{cccc}
			-\textbf{p}^f_{y}\\ \textbf{p}^f_{x}
		\end{array}
		\right]
		\end{array}
		\right]\\
  \underrightarrow{{\substack{\text{L}_3-\left[
		\begin{array}{cccc}
			-\textbf{p}^r_{y}\\ \textbf{p}^r_{x}\\0
		\end{array}
		\right]\cdot\text{L}_2\\\text{L}_5-\left[
		\begin{array}{cccc}
			-\textbf{p}^f_{y}\\ \textbf{p}^f_{x}
		\end{array}
		\right]\cdot\text{L}_2}}}
  \left[
		\begin{array}{cccccc}
			\textbf{0}_{2\times 3}& 
			\textbf{0}_{2\times 1} \\ \textbf{0}_{1\times 3}&1\\
			\textbf{I}_3 
			& \textbf{0}_{3\times 1}\\ \left[\begin{array}{cccccc}\textbf{0}_{1\times 2} &1\end{array}\right]&0\vspace{1mm}\\
			\left[\begin{array}{cccccc}\textbf{I}_{2} &\textbf{0}_{2\times 1}\end{array}\right]& \textbf{0}_{2\times 1}
		\end{array}
		\right],
  \end{array}
\end{equation}
we can obtain an affine atlas,
\begin{equation}\label{S2_conpoint_AffAtlas}
    \mathcal A^\xi=\{\psi_{\hat{\textbf{X}}}=\textbf{A}_{\hat{\textbf{X}}}\cdot \phi_{\hat{\textbf{X}}}| \phi_{\hat{\textbf{X}}}\in\mathcal A^\eta\},
\end{equation}
where
$$\textbf{A}_{\textbf{X}}=\left[
		\begin{array}{ccccccc}
			\textbf{I}_2& \textbf{0}_{2\times 1}& \textbf{0}_{2\times 3}&\textbf{0}_{2\times 1}&\textbf{0}_{2\times 2}\\
   \textbf{0}_{1\times 2}& 1& \textbf{0}_{1\times 3}&0&\textbf{0}_{1\times 2}\\
   \textbf{0}_{3\times 2}&
			\left[
		\begin{array}{ccccccc}
			\textbf{p}^r_{y}\\ -\textbf{p}^r_{x}\\0
		\end{array}
		\right]& \textbf{I}_3& \textbf{0}_{3\times 1}& \textbf{0}_{3\times 2}\\
  \textbf{0}_{1\times 2}& 0& \textbf{0}_{1\times 3}&1&\textbf{0}_{1\times 2}\\
			\textbf{0}_{2\times 2}&
			\left[
		\begin{array}{cccccccc}
			\textbf{p}^f_{y}\\ -\textbf{p}^f_{x}
		\end{array}
		\right]& \textbf{0}_{2\times 3}&\textbf{0}_{2\times 1}&\textbf{I}_2
		\end{array}
		\right].$$
Then, by Theorem \ref{MT3}, the corresponding Aff-EKF of (\ref{S2_conpoint_AffAtlas}) satisfies the observability constraint for the considered problem.

In summary, the observability discrepancy appears in Std-EKFs for both scenarios of the considered problems. As a consequence, theoretically there exists inconsistency issue in Std-EKFs.
 
 As far as we know, currently, there is no invariant EKF applied to these constrained feature problems. In the RI-EKF for the typical 3D point features, the group operation of $\mathbb{SE}_{K+1}(3)$ cannot remain the feature constraints (the height of features are the same). As a result, the RI-EKF (\ref{chart_RI_pointSLAM}) for typical point features is not compatible for those constrained features.

 In contrast, based on the proposed Alg. \ref{Alg3:FindA}, we can easily obtain
the observability preserved Aff-EKFs with (\ref{S1_conpoint_AffAtlas}) if the height is known or with (\ref{S2_conpoint_AffAtlas}) if the height is unknown.

\textbf{Computational Complexity.} Suppose there are $K$ features in the state. Then, for the propagation, the cost of our Aff-EKF (Alg. \ref{EKF_Framework}) is $O(K^2)$, and the cost of our Aff-EKF (Alg. \ref{Alg2:Aff-EKF}) is $O(K)$. For the update, the costs of these two different implementations of our Aff-EKF are $O(K^2)$. In addition, there is an extra $O(K^2)$ for the modifications on covariance in Aff-EKF (Alg. \ref{Alg2:Aff-EKF}). As a comparison, the costs of Std-EKF are $O(K)$ for the propagation and $O(K^2)$ for the update.



\subsection{Plane Feature based SLAM}\label{Sec_PlaneSLAM}


In the following considered SLAM problem, the plane features, $\textbf{p}^f\in \mathbb{R}^3$, are represented by the following closest point form \cite{PointLinePlane}:
\begin{equation} \label{feature_plane}
    \textbf{p}^f=d^f\textbf{n}^f,
\end{equation}
where $d^f\geq 0$ is the shortest distance from origin to the plane and $\textbf{n}^f\in \mathbb{R}^3$ is a norm vector representing the direction of plane.
The states with $K$ plane features can be defined by
\begin{equation}
    \textbf{X}=(\textbf{R},\textbf{p}^r,\textbf{p}^{f_1},\cdots,\textbf{p}^{f_K})\in\mathbb{SO}(3)\times \mathbb{R}^3\times\mathbb{R}^{3K},
\end{equation}
where $\textbf{R}$ and $\textbf{p}^r$ represent the rotation and the position of the robot, respectively, and $\textbf{p}^{f_j} \in \mathbb{R}^3$ is the parameter of the $j$-th plane feature of the form (\ref{feature_plane}), all described in the global coordinate system. 

The process model (\ref{gen_mot}) is 
\begin{equation}\label{mot_PlaneSLAM}
\begin{array}{lll}
(\textbf{R}_{n},\textbf{p}^r_{n},\textbf{p}^{f}_{n})=(\textbf{R}_{n-1}{\textbf{R}}^u_{n-1},\textbf{p}^r_{n-1}+\textbf{R}_{n-1}{\textbf{p}}_{n-1}^u,\textbf{p}^{f}_{n-1}),
\end{array}    
\end{equation}
where ${\textbf{U}}_{n-1}=({\textbf{R}}^u_{n-1},{\textbf{p}}^u_{n-1})\in \mathbb{SO}(3)\times \mathbb{R}^3$ is the odometry. 
And according to \cite{PointLinePlane}, the observation model (\ref{gen_obs}) is given by
\begin{equation}\label{obs_PlaneSLAM}
\begin{array}{lll}
    \textbf{z}_{n}
    &=(d^f_{n}-(\textbf{p}^r_{n})^\top\textbf{n}^f_{n})\textbf{R}^\top_{n}\textbf{n}^f_{n}.
\end{array}    
\end{equation}

We can easily find an atlas of the form (\ref{example_std}), which is regarded as the standard atlas $\mathcal A^\eta$. Then, the corresponding Jacobian matrices by the atlas $\mathcal A^\eta$ are shown below.
\begin{equation}
\begin{array}{llll}
     \textbf{F}_{n-1}^{\eta}&=\left[
		\begin{array}{cccc}
			\textbf{I}_3& \textbf{0}_{3\times 3}& \textbf{0}_{3\times 3}\\
		-(\textbf{p}_{n}^{r}-\textbf{p}_{n-1}^{r})^{\land}&  \textbf{I}_3& \textbf{0}_{3\times 3}\\
			\textbf{0}_{3\times 3}&  \textbf{0}_{3\times 3}& \textbf{I}_3
		\end{array}
		\right],\vspace{1mm}\\
     \textbf{G}_{n-1}^{\eta}&=\left[
		\begin{array}{cccc}
			\textbf{R}_{n-1}&  \textbf{0}_{3\times 3}\\
			\textbf{0}_{3\times 3}&  \textbf{R}_{n-1}\\
			\textbf{0}_{3\times 3}&  \textbf{0}_{3\times 3}
		\end{array}
		\right],\vspace{1mm}\\
     \textbf{H}_n^{\eta}&=\textbf{R}_{n}^{{\top}}\left[\begin{array}{cccc}
     \textbf{H}_{n,1} & \textbf{H}_{n,2} & \textbf{H}_{n,3}\end{array}
		\right],\\
     \end{array}
\end{equation}
where
$$\textbf{H}_{n,1}=(d^f-(\textbf{p}_{n}^r)^\top\textbf{n}^f)(\textbf{n}^f)^\land,$$
$$\textbf{H}_{n,2}=-\textbf{n}^f(\textbf{n}^f)^\top,$$
$$\textbf{H}_{n,3}=\frac{(d^f-(\textbf{p}_{n}^r)^\top\textbf{n}^f)\textbf{I}_3-\textbf{n}^f(\textbf{p}_{n}^r)^\top+2(\textbf{n}^f)^\top\textbf{p}_{n}^r\textbf{n}^f(\textbf{n}^f)^\top}{d^f}.$$

Based on the standard atlas $\mathcal A^\eta$, the unobservable subspace of the underlying dynamic system at the state $\textbf{X}$ is obtained as
\begin{equation}
		{\textbf{N}}^{\eta}(\textbf{X})=\mathop{\text{span}} _{col.}\left[
		\begin{array}{cccccc}
			\textbf{0}_{3\times3}& \textbf{I}_3\\\textbf{I}_3
			& -(\textbf{p}^r)^\land\\
			\textbf{n}^f(\textbf{n}^f)^\top& -d^f(\textbf{n}^f)^\land
		\end{array}
		\right],
	\end{equation}
which obviously depends on the state values. Therefore, by Theorem \ref{MT1}, the corresponding Std-EKF does not satisfy the observability constraint.

To the best of our knowledge, there is currently no RI-EKF or other EKFs that can naturally preserve the correct observability property for this considered problem. However, through the procedure (Alg. \ref{Alg3:FindA}),
\begin{equation}
    \left[
		\begin{array}{cccccc}
			\textbf{0}_{3\times3}& \textbf{I}_3\\\textbf{I}_3
			& -(\textbf{p}^r)^\land\\
			\textbf{n}^f(\textbf{n}^f)^\top& -d^f(\textbf{n}^f)^\land
		\end{array}
		\right]\underrightarrow{\substack{\text{L}_2+(\textbf{p}^r)^\land)\cdot\text{L}_1\\\text{L}_3+d^f(\textbf{n}^f)^\land)\cdot\text{L}_1\\
  \text{L}_3-\textbf{n}^f(\textbf{n}^f)^\top)\cdot\text{L}_2}}
  \left[
		\begin{array}{cccccc}
			\textbf{0}_{3\times3}& \textbf{I}_3\\\textbf{I}_3
			& \textbf{0}_{3\times3}\\
			\textbf{0}_{3\times3}& \textbf{0}_{3\times3}
		\end{array}
		\right],
\end{equation}
we can obtain an affine atlas,
\begin{equation}\label{plane_AffAtlas}
    \mathcal A^\xi=\{\psi_{\hat{\textbf{X}}}=\textbf{A}_{\hat{\textbf{X}}}\cdot \phi_{\hat{\textbf{X}}}| \phi_{\hat{\textbf{X}}}\in\mathcal A^\eta\},
\end{equation}
where
\begin{equation}
    \textbf{A}_{\textbf{X}}
  =\left[\begin{array}{cccccc}
		\textbf{I}_3 & \textbf{0}_{3\times3}&\textbf{0}_{3\times3}\\
            (\textbf{p}^r)^\land&\textbf{I}_3&\textbf{0}_{3\times3}\\
            d^f(\textbf{n}^f)^\land-\textbf{n}^f(\textbf{n}^f)^\top(\textbf{p}^r)^\land&-\textbf{n}^f(\textbf{n}^f)^\top&\textbf{I}_3
		\end{array}\right].
\end{equation}
Then, by Theorem \ref{MT3}, the corresponding Aff-EKF generated by (\ref{plane_AffAtlas}) naturally maintains the correct observability property for the considered plane feature SLAM problems.

\textbf{Computational Complexity.} Suppose there are $K$ features in the state. Then, for the propagation, the cost of our Aff-EKF (Alg. \ref{EKF_Framework}) is $O(K^2)$, and the cost of our Aff-EKF (Alg. \ref{Alg2:Aff-EKF}) is $O(K)$. For the update, the costs of these two different implementations of our Aff-EKF are $O(K^2)$. In addition, there is an extra $O(K^2)$ for the modifications on covariance in Aff-EKF (Alg. \ref{Alg2:Aff-EKF}). As a comparison, the costs of Std-EKF are $O(K)$ for the propagation and $O(K^2)$ for the update.


%


\section{Experiments}\label{Sec_Results}

\begin{table*}[tp]
	
	\centering
	\caption{Monte-Carlo Simulation Parameters}
	\begin{tabular}{|c|c|c|c|c|c|c|c|c|c|}
		\hline
Environment ID  & 1& 2& 3& 4& 5\\ \hline
Type of Features& Typical Point& Constrained Point& Constrained Point& Plane& Plane\\
Number of Features& 50& 40& 40& 14& 10\\
Trajectory Length& 473.32m& 942.44m& 473.31m& 707.60m& 471.76m\\
Average Odometry Orientation & 0.020rad& 0.031rad& 0.029rad& 0.027rad& 0.016rad\\
Average Odometry Position& 0.24m& 0.94m& 0.47m& 0.35m& 0.24m\\
Average Observation Measurement & 2.95m& 2.81m& 2.83m& 2.19m& 2.32m\\
		\hline
	\end{tabular}\label{Para_Sim}
	
\end{table*}

In this section, we evaluate the performance of the proposed Aff-EKF methods by Monte Carlo simulations for the 3 applications introduced in Sec. \ref{Sec_Applications}. All the implementations are in Matlab, and all computation is performed on the same laptop (AMD Ryzen 7 4800U with Radeon Graphics 1.80 GHz, 16 GB GB RAM).

In each experiment, the covariance matrices of odometry noise and observation noise in (\ref{gen_mot}) and (\ref{gen_obs}) are set to be
$\bm{\Sigma}_{n-1}=\text{diag}(\sigma_{w_1}^2 \textbf{I}_3, \sigma_{w_2}^2 \textbf{I}_3)$ and $\bm{\Omega}^{f_i}_n=\sigma_{v}^2 \textbf{I}_3$. Therefore, for concise presentation of the article, we use $(\sigma_{w_1}, \sigma_{w_2}, \sigma_{v})$ to represent such noise setting.

The detailed parameters of the simulation environments are shown in Tab. \ref{Para_Sim}. These environments are also shown in the first column of corresponding figures of the experiments below.

We adopt the root mean square error (RMSE) and normalized estimation
error squared (NEES) indicators to evaluate the accuracy and the consistency of the EKF algorithms, respectively. The definitions of these indicators can be found in \cite{Obj1,Zhuqing_IEKFlocalization}. Specifically, the value of NEES should be around 1 through the Monte Carlo simulation, if the estimator is consistent. And the NEES much greater than 1 means that the system underestimates the covariance of the state \cite{Obj1,Zhuqing_IEKFlocalization}.
In addition, similar to \cite{Obj1,Zhuqing_IEKFlocalization}, RMSE of all the estimators are computed in the standard atlas (using the same error form as Std-EKF) for a fair comparison.




\begin{figure*}[htp]
\centering 
 \includegraphics[width=1\textwidth]{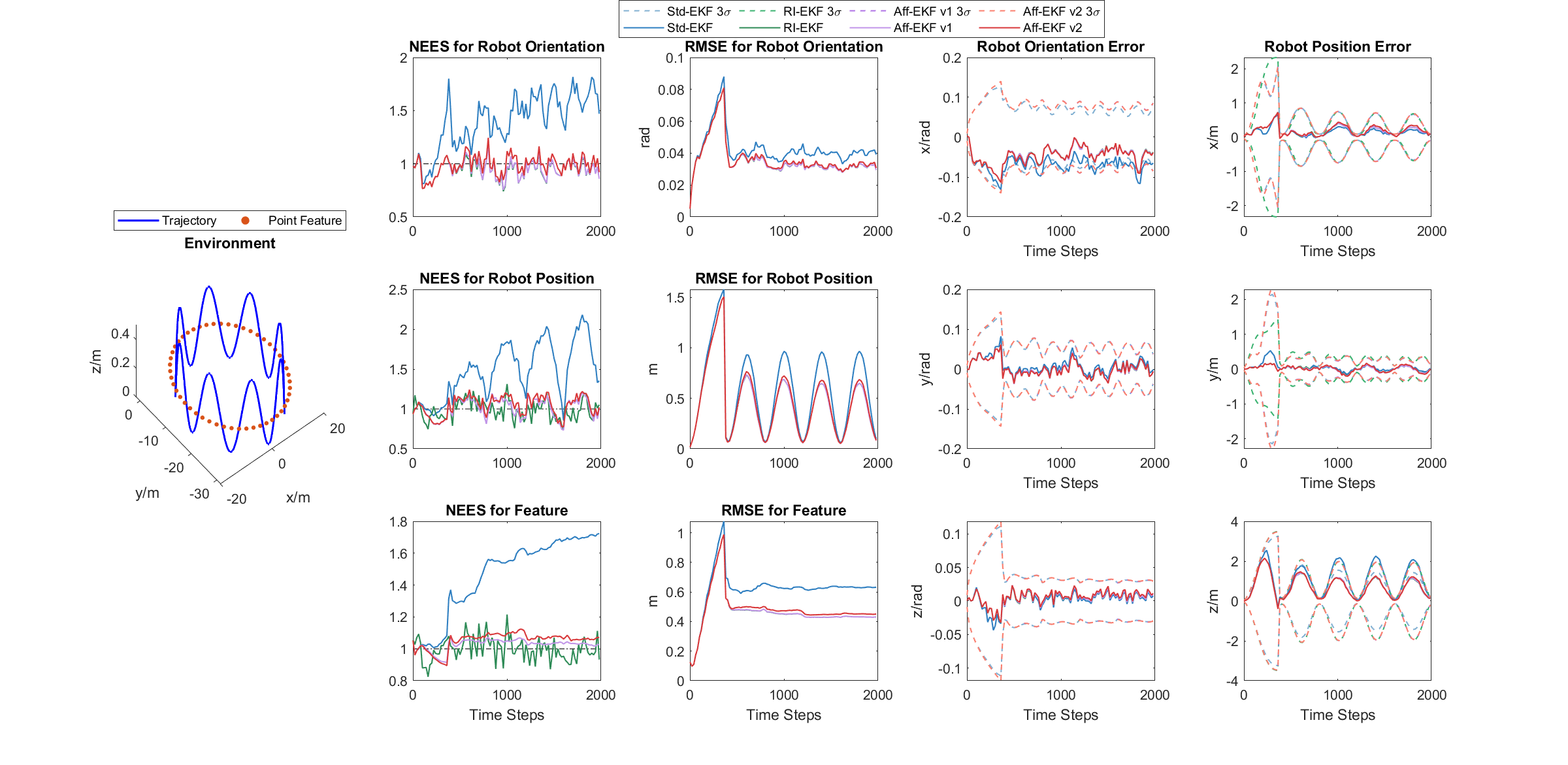} 
	\caption{Monte Carlo Results of SLAM with Point Features in Environment 1 with Noise Settings of $(0.003,0.01,0.1)$}
	\label{ResultsFig_Point}
\end{figure*}

\begin{table*}[htp]
	
	\centering
	\caption{Monte Carlo Results of SLAM with Point Features (50 runs).}
	\begin{tabular}{|c|c|c|c|c|c|c|c|c|c|}
		\hline
  
		\textbf{Noises}&	\multirow{2}*{\textbf{Algorithms}}&	\multicolumn{3}{|c|}{\textbf{RMSE}}  &\multicolumn{2}{|c|}{\textbf{NEES}}&\multirow{2}*{\textbf{Time (s)}}\\ 
        \cline{3-7}
		($\sigma_{w_1}$, $\sigma_{w_2}$, $\sigma_{v}$)&	~&	 \textbf{Rob. Ori. (rad)}&	 \textbf{Rob. Pos. (m)}&	 \textbf{Fea. (m)}&	 \textbf{Rob. Pose}&	 \textbf{Fea.} & ~\\
		\hline
  \multicolumn{9}{|c|}{\textbf{Environment 1 (with Point Features)}}\\
  \hline
\multirow{3}*{(0.003,0.01,0.1)}&Std-EKF&0.0427&0.5706&0.6159&1.294&1.464&16.957\\
~&FEJ-EKF&0.0377&0.4719&0.4891&1.075&1.119&17.262\\
~&OC-EKF&0.0396&0.5199&0.5521&1.126&1.216&23.880\\
~&FEJ2-EKF&0.0377&0.4723&0.4799&0.989&1.149&17.594\\
~&DRI-EKF&0.0377&0.4725&0.4898&1.076&1.117&11.638\\
~&RI-EKF&{\color{red}\textbf{0.0362}}$^*$&{\color{red}\textbf{0.4517}}&{\color{red}\textbf{0.4646}}&0.991&1.007&15.193\\
~&Aff-EKF v1&{\color{red}\textbf{0.0362}}&{\color{blue}\textbf{0.4520}}&{\color{blue}\textbf{0.4651}}&0.998&1.028&14.769\\
~&Aff-EKF v2&{\color{blue}\textbf{0.0368}}$^*$&0.4646&0.4800&1.025&1.054&23.605\\\hline
\multirow{3}*{(0.003,0.01,0.2)}&Std-EKF&0.0480&0.6215&0.6505&1.165&1.213&23.108\\
~&FEJ-EKF&0.0477&0.6098&0.6416&1.117&1.120&23.184\\
~&OC-EKF&0.0474&0.6113&0.6425&1.102&1.119&32.783\\
~&FEJ2-EKF&0.0478&0.6100&0.6350&1.033&1.103&23.381\\
~&DRI-EKF&0.0477&0.6098&0.6416&1.114&1.118&15.365\\
~&RI-EKF&{\color{blue}\textbf{0.0460}}&{\color{blue}\textbf{0.5830}}&{\color{blue}\textbf{0.6091}}&1.024&0.999&20.640\\
~&Aff-EKF v1&{\color{blue}\textbf{0.0460}}&0.5831&{\color{blue}\textbf{0.6091}}&1.027&1.020&20.237\\
~&Aff-EKF v2&{\color{red}\textbf{0.0459}}&{\color{red}\textbf{0.5822}}&{\color{red}\textbf{0.6077}}&1.036&1.032&31.569\\\hline
\multirow{3}*{(0.003,0.02,0.1)}&Std-EKF&0.0441&0.5947&0.6218&1.365&1.562&29.974\\
~&FEJ-EKF&0.0399&0.5108&0.5198&1.087&1.087&29.605\\
~&OC-EKF&0.0412&0.5413&0.5599&1.133&1.198&40.832\\
~&FEJ2-EKF&0.0403&0.5313&0.5258&1.014&1.134&29.814\\
~&DRI-EKF&0.0398&0.5090&0.5179&1.089&1.085&19.447\\
~&RI-EKF&{\color{red}\textbf{0.0389}}&{\color{blue}\textbf{0.5036}}&{\color{blue}\textbf{0.5115}}&1.015&1.000&25.554\\
~&Aff-EKF v1&{\color{red}\textbf{0.0389}}&{\color{red}\textbf{0.5032}}&{\color{red}\textbf{0.5110}}&1.021&1.035&25.042\\
~&Aff-EKF v2&{\color{blue}\textbf{0.0391}}&0.5053&0.5135&1.052&1.062&40.529\\\hline
	\end{tabular}
 \label{ResultsTab_Point}
\\ [3bp]
$ \begin{array}{l}
	    ^*\text{{\color{red}\textbf{Red}} and {\color{blue}\textbf{blue}} for RMSE indicate the best and second best results, respectively.}\\ 
     \text{NEES values are incomparable if they are around 1.}
	\end{array}$	
\end{table*}

\subsection{Point Feature based SLAM}

For the typical point feature based SLAM, we take various EKF frameworks into comparison. The corresponding EKF algorithms include Std-EKF, FEJ-EKF \cite{con4}, OC-EKF \cite{con4}, FEJ2-EKF \cite{FEJ2}, DRI-EKF \cite{DRI}, RI-EKF \cite{con3}, and our proposed two versions of Aff-EKFs generated by (\ref{3Dpoint_AffAtlas1}) and (\ref{3Dpoint_AffAtlas2}), respectively. 

The experiments are conducted in Environment 1 shown in Tab. \ref{Para_Sim} and Fig. \ref{ResultsFig_Point}. The average RMSE and NEES of the whole
trajectory for different noise levels are presented in Tab. \ref{ResultsTab_Point}. The results show that all the EKF methods except Std-EKF can produce consistent estimates. However, RI-EKF and our proposed two Aff-EKFs usually outperform other methods in terms of accuracy. 

The RMSE and NEES of Std-EKF, RI-EKF and two proposed Aff-EKFs are also displayed in Fig. \ref{ResultsFig_Point}. Both RI-EKF and two Aff-EKFs remain consistent during the whole trajectory, while Std EKF becomes apparently inconsistent after a period. In particular, Fig. \ref{ResultsFig_Point} also shows robot orientation errors and robot position errors with $3\sigma$ bounds from these four EKFs in one Monte Carlo simulation. We can find that the errors of Std-EKF are outside the corresponding $3\sigma$ bounds (99\% envelope), which clearly reveals the inconsistent issue of Std-EKF. 

In summary, the above experiments illustrate that our proposed framework greatly improves the performance of Std-EKF, and the corresponding two Aff-EKFs achieve the performance as good as RI-EKF.

\subsection{SLAM with Horizontal Plane Constrained Point Features}

In this part, we compare our Aff-EKF with Std-EKF and Ideal-EKF (a variant of the
Std-EKF where Jacobians are evaluated at the ground-truth) for these constrained point feature based SLAM problems on Environment 2 and 3 described in Tab. \ref{Para_Sim}. In addition, we also take RI-EKF of typical point SLAM into comparison, as there is no current I-EKF available for this constraint point situation. In order to execute, this RI-EKF ignores the plane constraints on the point features.

The average RMSE and NEES under different noise levels are listed in Tab. \ref{ResultsTab_ConPointKnown} and Tab. \ref{ResultsTab_ConPointUnKnown} for scenarios 1 and 2, respectively. In particular, under the noise level of $(0.005,0.01,0.1)$, Fig. \ref{ResultsFig_ConPointKnown} and Fig. \ref{ResultsFig_ConPointUnKnown} also display the RMSE and NEES of robot estimates as well as the $3\sigma$ bounds in one simulation for scenario 1 (height is known) and scenario 2 (height is unknown), respectively.

In both Tab. \ref{ResultsTab_ConPointKnown} and Tab. \ref{ResultsTab_ConPointUnKnown}, the average NEES values of RI-EKF, Ideal-EKF and our Aff-EKF are very close to 1, while that of Std-EKF is obviously larger than 1. Fig. \ref{ResultsFig_ConPointKnown} and Fig. \ref{ResultsFig_ConPointUnKnown} also show that for all of these methods, except for the standard EKF, the NEES values are approximately 1 throughout the entire trajectory. And we can also find that robotic errors of Std-EKF exceed the corresponding $3\sigma$ bounds, while those of RI-EKF, Ideal-EKF, and our Aff-EKF remain well within them. 
The above results indicate that Std-EKF is overconfident in its uncertainty estimates for the considered problems. In contrast, both RI-EKF and our Aff-EKF have good consistency property similar to Ideal-EKF. 

However, we can notice that in both Tab. \ref{ResultsTab_ConPointKnown} and Tab. \ref{ResultsTab_ConPointUnKnown}, RI-EKF has the maximum average RMSE for robot orientation and position, as well as feature positions. Our Aff-EKF has less estimation errors than RI-EKF and Std-EKF, and performs close to Ideal-EKF. The presentations of RMSE in Fig. \ref{ResultsFig_ConPointKnown} and Fig. \ref{ResultsFig_ConPointUnKnown} also show that the errors in robot pose estimation of RI-EKF that ignores the plane constraint are significantly greater than other three EKF methods that utilize such constraint. 

Therefore, due to the lack of plane constraint, typical RI-EKF cannot even compare to Std-EKF in terms of accuracy, although it is more consistent. In contrast, our proposed Aff-EKF outperforms Std-EKF in both accuracy and consistency, and achieves performance close to Ideal-EKF.





\begin{figure*}[ht]
\centering 
  \begin{subfigure}{0.95\linewidth} 
  \includegraphics[width=0.95\textwidth]{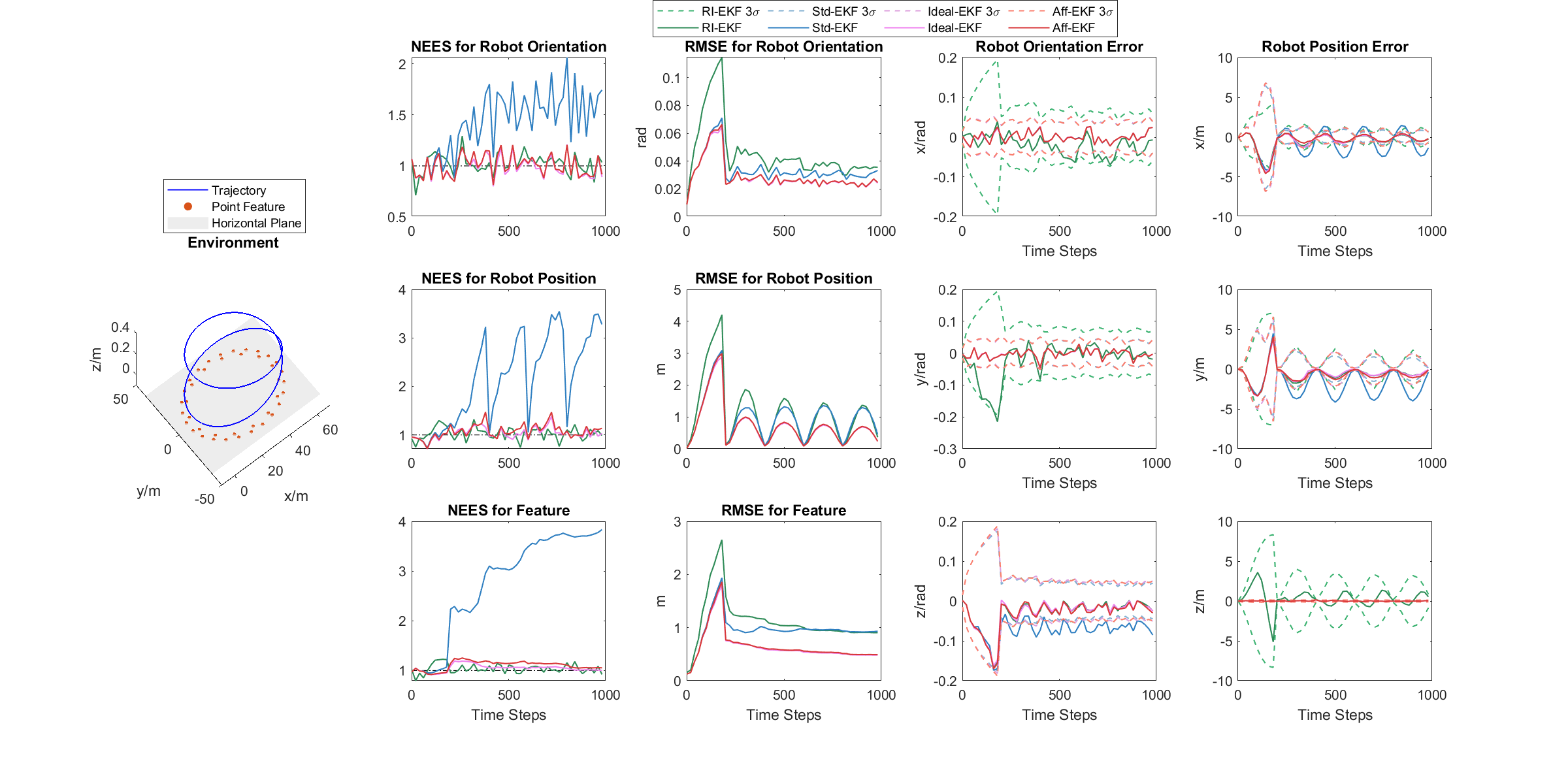}
  \caption{Environment 2 with Noise Settings of $(0.005, 0.01, 0.1)$}\vspace{5mm} \label{ResultsFig_ConP1_Data1}  \end{subfigure}
  
 \begin{subfigure}{0.95\linewidth} \includegraphics[width=0.95\textwidth]{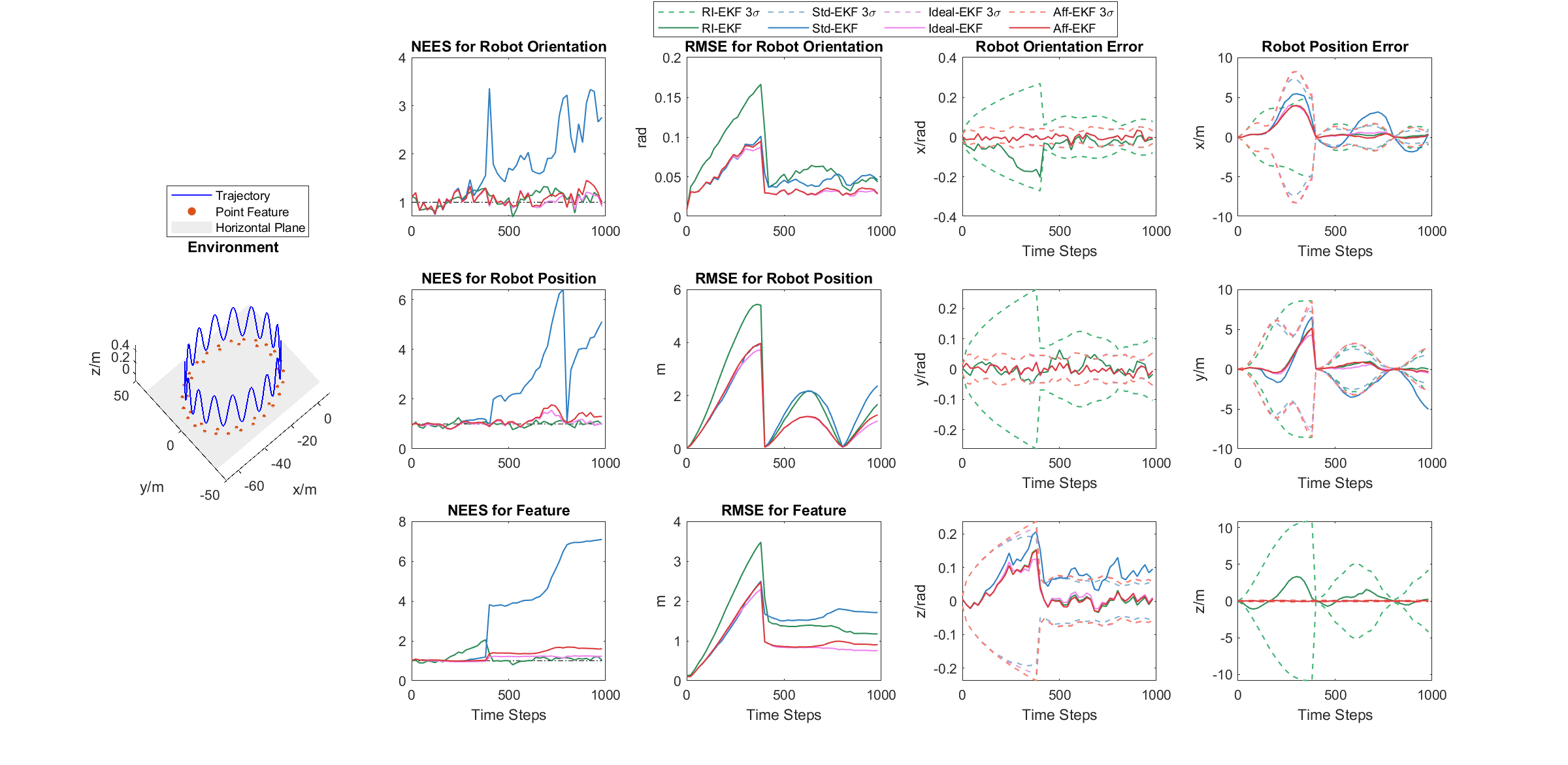} \caption{Environment 3 with Noise Settings of $(0.005, 0.01, 0.1)$}\label{ResultsFig_ConP1_Data2} \end{subfigure}
	\caption{Monte Carlo Results of SLAM with Known Height Horizontal Plane Constrained Point Features}
	\label{ResultsFig_ConPointKnown}
\end{figure*}

\begin{table*}[ht]
	
	\centering
	\caption{Monte Carlo Results of SLAM with Known Height Horizontal Plane Constrained Point Features (50 runs).}
	\begin{tabular}{|c|c|c|c|c|c|c|c|c|c|}
		\hline
  
		\textbf{Noises}&	\multirow{2}*{\textbf{Algorithms}}&	\multicolumn{3}{|c|}{\textbf{RMSE}}  &\multicolumn{2}{|c|}{\textbf{NEES}}&\multirow{2}*{\textbf{Time (s)}}\\ 
        \cline{3-7}
		($\sigma_{w_1}$, $\sigma_{w_2}$, $\sigma_{v}$)&	~&	 \textbf{Rob. Ori. (rad)}&	 \textbf{Rob. Pos. (m)}&	 \textbf{Fea. (m)}&	 \textbf{Rob. Pose}&	 \textbf{Fea.} & ~\\
		\hline
  \multicolumn{9}{|c|}{\textbf{Environment 2 (with Known Height Horizontal Plane Constrained Point Features)}}\\
  \hline
\multirow{3}*{(0.005,0.01,0.1)}&RI-EKF&0.0464&1.1715&1.1171&1.026&1.035&2.453\\
~&Std-EKF&0.0345&0.9945&0.9645&1.612&2.797&1.752\\
~&Ideal-EKF&{\color{red}\textbf{0.0295}}&{\color{red}\textbf{0.7178}}&{\color{red}\textbf{0.6419}}&1.014&1.049&1.774\\
~&Aff-EKF&{\color{blue}\textbf{0.0298}}&{\color{blue}\textbf{0.7329}}&{\color{blue}\textbf{0.6547}}&1.036&1.106&2.487\\\hline
\multirow{3}*{(0.005,0.01,0.15)}&RI-EKF&0.0502&1.2711&1.2101&1.031&1.027&2.455\\
~&Std-EKF&0.0385&1.0752&1.0243&1.453&2.302&1.885\\
~&Ideal-EKF&{\color{red}\textbf{0.0333}}&{\color{red}\textbf{0.7501}}&{\color{red}\textbf{0.6501}}&1.006&0.992&2.653\\
~&Aff-EKF&{\color{blue}\textbf{0.0338}}&{\color{blue}\textbf{0.7791}}&{\color{blue}\textbf{0.6812}}&1.034&1.078&2.653\\\hline
\multirow{3}*{(0.005,0.01,0.2)}&RI-EKF&0.0553&1.3731&1.3035&1.025&1.041&2.476\\
~&Std-EKF&0.0406&1.1028&1.0466&1.407&2.162&2.630\\
~&Ideal-EKF&{\color{red}\textbf{0.0354}}&{\color{red}\textbf{0.7600}}&{\color{red}\textbf{0.6567}}&0.974&0.908&1.886\\
~&Aff-EKF&{\color{blue}\textbf{0.0358}}&{\color{blue}\textbf{0.7882}}&{\color{blue}\textbf{0.6840}}&0.990&0.962&2.630\\\hline
\multicolumn{9}{|c|}{\textbf{Environment 3 (with Known Height Horizontal Plane Constrained Point Features)}}\\
  \hline
\multirow{3}*{(0.005,0.01,0.1)}&RI-EKF&0.0717&1.8512&1.4981&1.029&1.162&0.985\\
~&Std-EKF&0.0517&1.6490&1.4682&1.845&3.622&0.596\\
~&Ideal-EKF&{\color{red}\textbf{0.0421}}&{\color{red}\textbf{1.2300}}&{\color{red}\textbf{0.9588}}&1.044&1.130&0.584\\
~&Aff-EKF&{\color{blue}\textbf{0.0436}}&{\color{blue}\textbf{1.2908}}&{\color{blue}\textbf{1.0305}}&1.095&1.318&0.824\\\hline
\multirow{3}*{(0.005,0.01,0.15)}&RI-EKF&0.0790&2.1399&1.7385&1.070&1.128&0.794\\
~&Std-EKF&0.0613&2.0519&1.8091&2.033&4.142&0.389\\
~&Ideal-EKF&{\color{red}\textbf{0.0462}}&{\color{red}\textbf{1.3476}}&{\color{red}\textbf{1.0400}}&1.034&1.039&0.569\\
~&Aff-EKF&{\color{blue}\textbf{0.0476}}&{\color{blue}\textbf{1.4198}}&{\color{blue}\textbf{1.0993}}&1.043&1.085&0.569\\\hline
\multirow{3}*{(0.005,0.01,0.2)}&RI-EKF&0.0780&2.0107&1.6663&1.001&1.084&0.805\\
~&Std-EKF&0.0504&1.5265&1.3218&1.256&1.789&0.574\\
~&Ideal-EKF&{\color{red}\textbf{0.0445}}&{\color{red}\textbf{1.1912}}&{\color{red}\textbf{0.9619}}&0.969&0.901&0.388\\
~&Aff-EKF&{\color{blue}\textbf{0.0458}}&{\color{blue}\textbf{1.2587}}&{\color{blue}\textbf{1.0208}}&0.990&0.975&0.574\\\hline
	\end{tabular}\label{ResultsTab_ConPointKnown}
	
\end{table*}



\begin{figure*}[ht]
\centering 
  \begin{subfigure}{1\linewidth} 
  \includegraphics[width=0.95\textwidth]{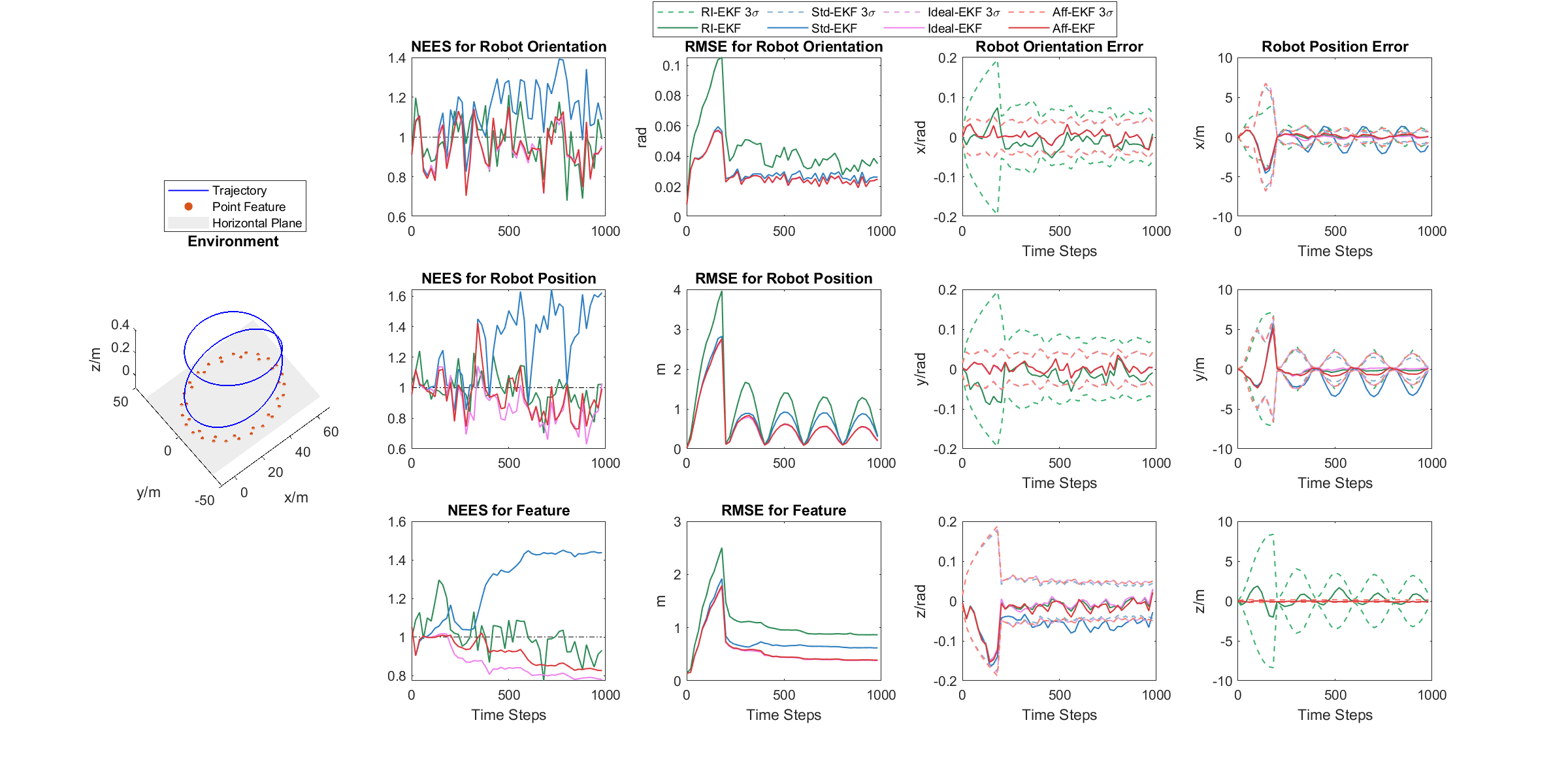}
  \caption{Environment 2 with Noise Settings of $(0.005, 0.01, 0.1)$}\vspace{5mm} \label{ResultsFig_ConP2_Data1}  \end{subfigure}
  
 \begin{subfigure}{1\linewidth} \includegraphics[width=0.95\textwidth]{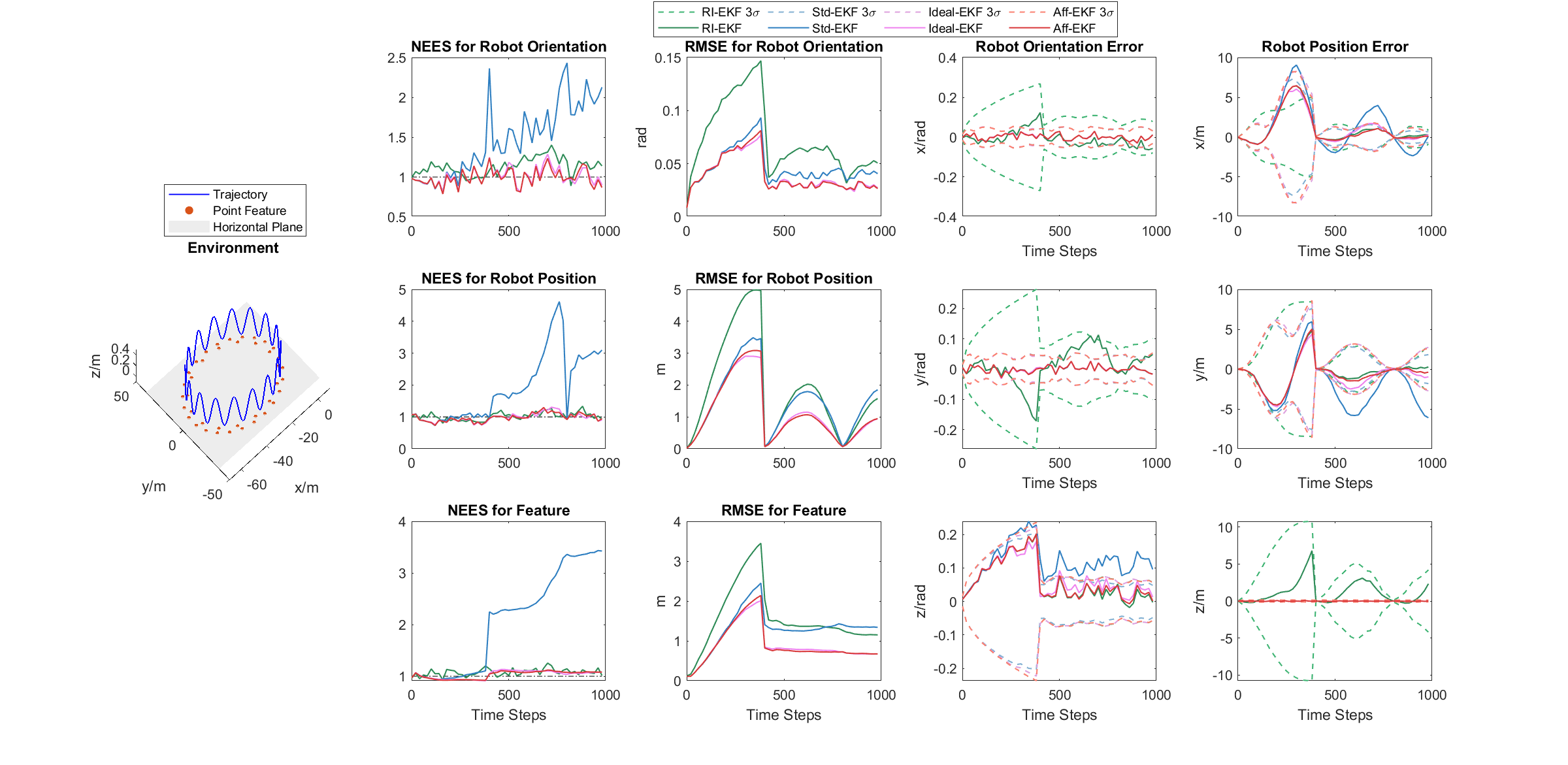} \caption{Environment 3 with Noise Settings of $(0.005, 0.01, 0.1)$}\label{ResultsFig_ConP2_Data2} \end{subfigure}
	\caption{Monte Carlo Results of SLAM with Unknown Height Horizontal Plane Constrained Point Features}
	\label{ResultsFig_ConPointUnKnown}
\end{figure*}

\begin{table*}[ht]
	
	\centering
	\caption{Monte Carlo Results of SLAM with Unknown Height Horizontal Plane Constrained Point Features (50 runs).}
	\begin{tabular}{|c|c|c|c|c|c|c|c|c|c|}
		\hline
  
		\textbf{Noises}&	\multirow{2}*{\textbf{Algorithms}}&	\multicolumn{3}{|c|}{\textbf{RMSE}}  &\multicolumn{2}{|c|}{\textbf{NEES}}&\multirow{2}*{\textbf{Time (s)}}\\ 
        \cline{3-7}
		($\sigma_{w_1}$, $\sigma_{w_2}$, $\sigma_{v}$)&	~&	 \textbf{Rob. Ori. (rad)}&	 \textbf{Rob. Pos. (m)}&	 \textbf{Fea. (m)}&	 \textbf{Rob. Pose}&	 \textbf{Fea.} & ~\\
		\hline
  \multicolumn{9}{|c|}{\textbf{Environment 2 (with Unknown Height Horizontal Plane Constrained Point Features)}}\\
  \hline
\multirow{3}*{(0.005,0.01,0.1)}&RI-EKF&0.0458&1.1087&1.0586&0.995&1.019&2.278\\
~&Std-EKF&0.0302&0.7922&0.7409&1.157&1.275&1.368\\
~&Ideal-EKF&{\color{red}\textbf{0.0281}}&{\color{red}\textbf{0.6418}}&{\color{red}\textbf{0.5658}}&0.947&0.860&1.378\\
~&Aff-EKF&{\color{blue}\textbf{0.0282}}&{\color{blue}\textbf{0.6519}}&{\color{blue}\textbf{0.5751}}&0.975&0.915&1.954\\\hline
\multirow{3}*{(0.005,0.01,0.15)}&RI-EKF&0.0501&1.1977&1.1504&0.963&1.001&2.246\\
~&Std-EKF&0.0362&0.9694&0.9330&1.300&1.544&1.634\\
~&Ideal-EKF&{\color{red}\textbf{0.0327}}&{\color{red}\textbf{0.7495}}&{\color{red}\textbf{0.6823}}&0.959&0.881&2.315\\
~&Aff-EKF&{\color{blue}\textbf{0.0331}}&{\color{blue}\textbf{0.7731}}&{\color{blue}\textbf{0.7052}}&0.979&0.912&2.315\\\hline
\multirow{3}*{(0.005,0.01,0.2)}&RI-EKF&0.0531&1.3238&1.2920&1.001&0.998&2.157\\
~&Std-EKF&0.0375&0.9290&0.8788&1.185&1.328&2.248\\
~&Ideal-EKF&{\color{red}\textbf{0.0346}}&{\color{red}\textbf{0.7455}}&{\color{red}\textbf{0.6760}}&0.982&0.943&1.576\\
~&Aff-EKF&{\color{blue}\textbf{0.0349}}&{\color{blue}\textbf{0.7565}}&{\color{blue}\textbf{0.6843}}&0.998&0.965&2.248\\\hline
\multicolumn{9}{|c|}{\textbf{Environment 3 (with Unknown Height Horizontal Plane Constrained Point Features)}}\\
  \hline
\multirow{3}*{(0.005,0.01,0.1)}&RI-EKF&0.0728&1.8767&1.5534&1.051&1.066&1.010\\
~&Std-EKF&0.0470&1.4779&1.2930&1.523&2.097&0.613\\
~&Ideal-EKF&{\color{red}\textbf{0.0394}}&{\color{red}\textbf{1.1225}}&{\color{blue}\textbf{0.8985}}&0.999&1.035&0.600\\
~&Aff-EKF&{\color{blue}\textbf{0.0395}}&{\color{blue}\textbf{1.1266}}&{\color{red}\textbf{0.8938}}&1.000&1.031&0.853\\\hline
\multirow{3}*{(0.005,0.01,0.15)}&RI-EKF&0.0736&1.9739&1.6016&1.010&1.154&0.833\\
~&Std-EKF&0.0594&1.9993&1.7604&2.109&3.267&0.402\\
~&Ideal-EKF&{\color{red}\textbf{0.0456}}&{\color{red}\textbf{1.3592}}&{\color{red}\textbf{1.0662}}&1.038&1.016&0.581\\
~&Aff-EKF&{\color{blue}\textbf{0.0478}}&{\color{blue}\textbf{1.4584}}&{\color{blue}\textbf{1.1614}}&1.051&1.043&0.581\\\hline
\multirow{3}*{(0.005,0.01,0.2)}&RI-EKF&0.0804&2.2279&1.7829&1.095&1.169&0.838\\
~&Std-EKF&0.0550&1.7491&1.5362&1.458&1.956&0.571\\
~&Ideal-EKF&{\color{red}\textbf{0.0474}}&{\color{red}\textbf{1.3850}}&{\color{red}\textbf{1.1334}}&1.028&1.039&0.399\\
~&Aff-EKF&{\color{blue}\textbf{0.0487}}&{\color{blue}\textbf{1.4523}}&{\color{blue}\textbf{1.1816}}&1.038&1.072&0.571\\\hline
	\end{tabular}\label{ResultsTab_ConPointUnKnown}
	
\end{table*}

\subsection{Plane Feature based SLAM}

In this part, our Aff-EKF is only compared with Std-EKF and Ideal-EKF for the plane feature based SLAM problems on Environment 4 and 5 described in Tab. \ref{Para_Sim}, as there is no current I-EKF for these problems. The results are listed in Fig. \ref{ResultsFig_Plane} and Tab. \ref{ResultsTab_Plane}. In Tab. \ref{ResultsTab_Plane}, we can observe that only the average NEES values of Std-EKF are significantly greater than 1. And Std-EKF has the maximum average RMSE values. In contrast, Ideal-EKF and our proposed Aff-EKF perform the best in terms of average RMSE and NEES evaluation. Fig. \ref{ResultsFig_Plane} further visually displays that the inconsistency issue of Std-EKF is also evidenced for this considered plane feature based SLAM problems. 
However, our proposed Aff-EKF can produce consistent and more accurate estimates than Std-EKF. In terms of RMSE and NEES evaluation, it also achieved performance close to Ideal-EKF.



\begin{figure*}[ht]
\centering 
  \begin{subfigure}{1\linewidth} 
  \includegraphics[width=0.97\textwidth]{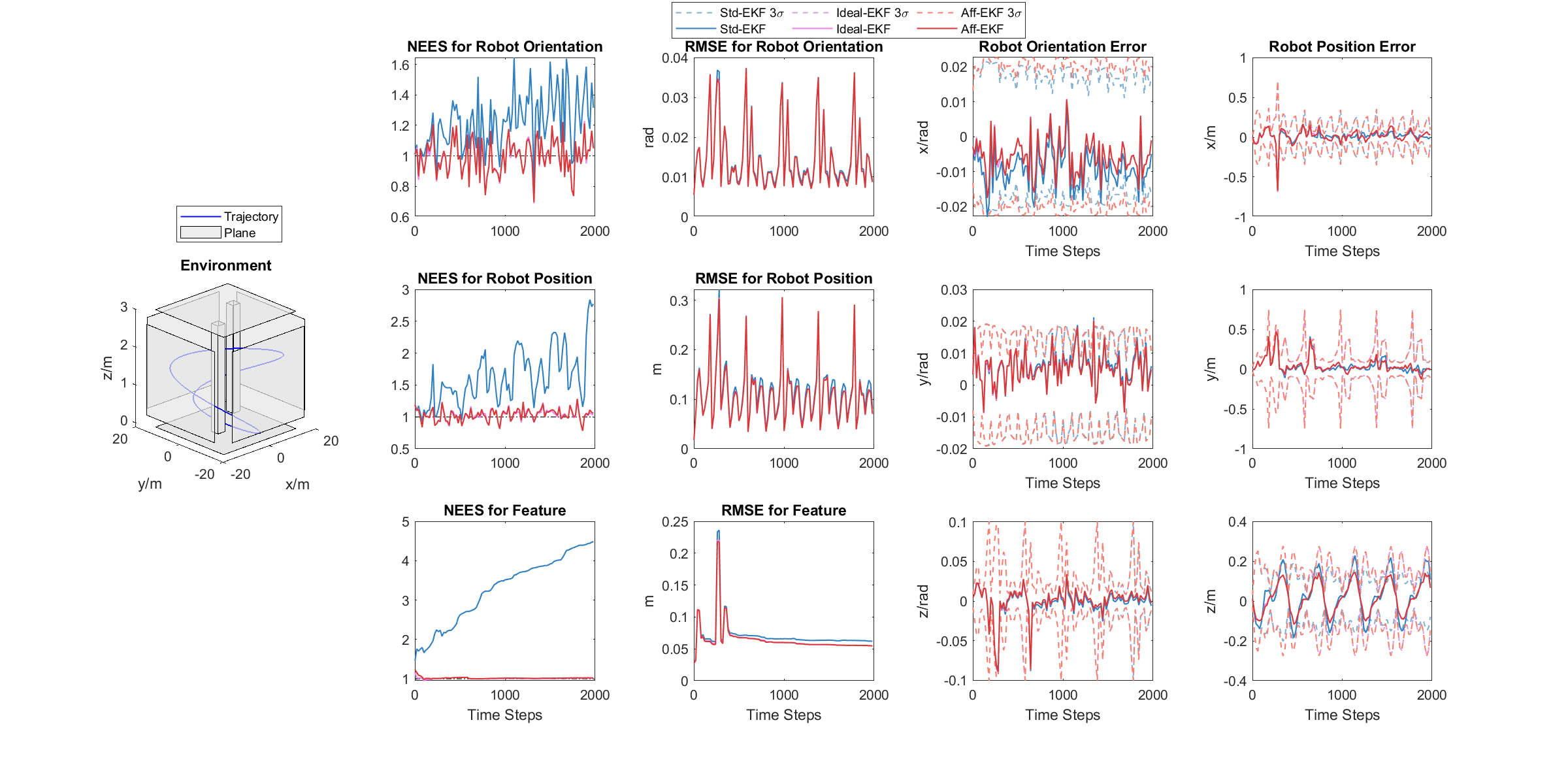}
  \caption{Environment 4 with Noise Settings of $(0.005, 0.01, 0.02)$}\vspace{5mm} \label{ResultsFig_Plane_Data1}  \end{subfigure}
  
 \begin{subfigure}{1\linewidth} \includegraphics[width=0.97\textwidth]{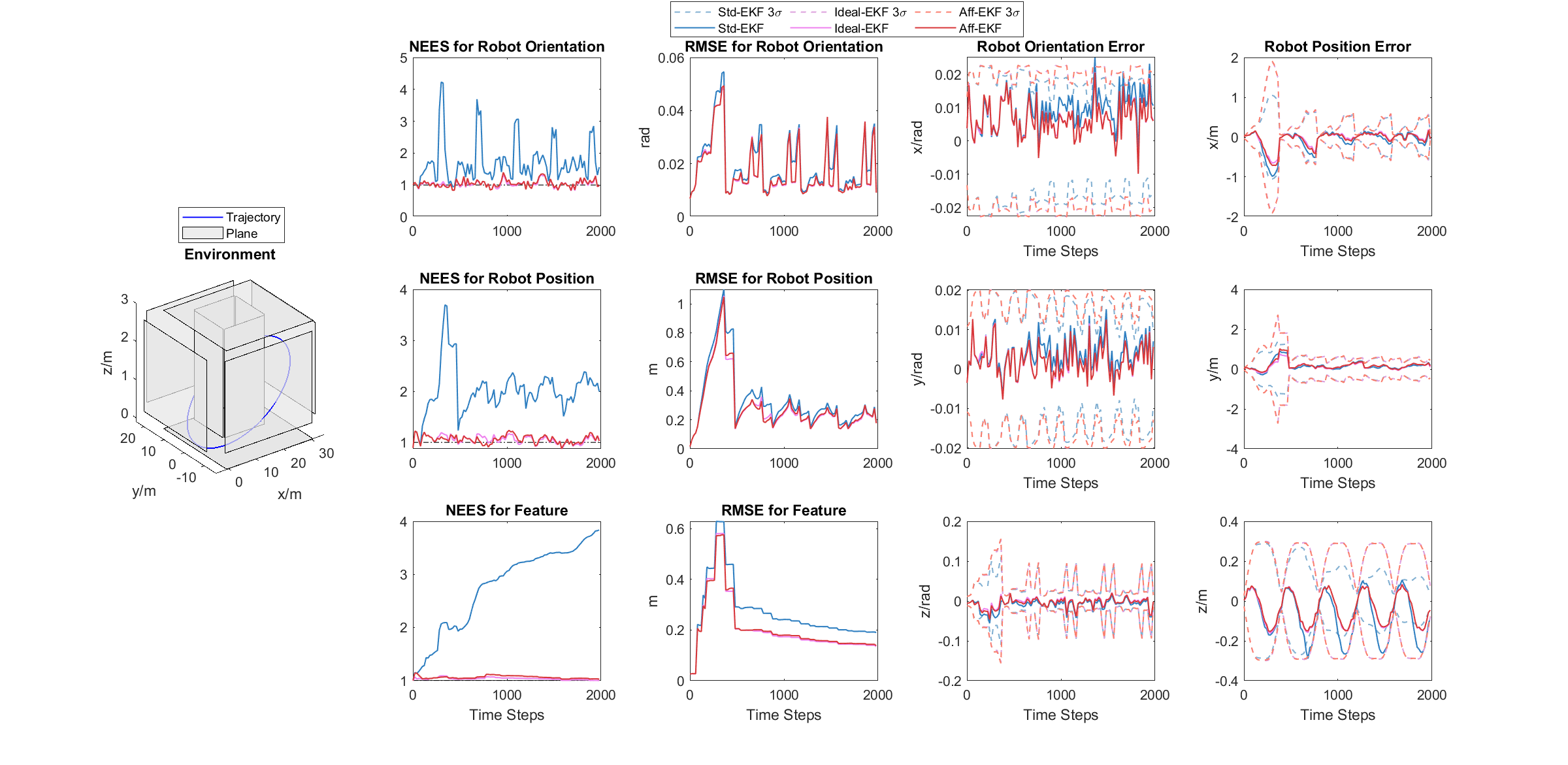} \caption{Environment 5 with Noise Settings of $(0.005, 0.01, 0.02)$}\label{ResultsFig_Plane_Data2} \end{subfigure}
	\caption{Monte Carlo Results of SLAM with Plane Features}
	\label{ResultsFig_Plane}
\end{figure*}





\begin{table*}[ht]
	
	\centering
	\caption{Monte Carlo Results of Environment 4 and 5 with Plane Features (50 runs).}
	\begin{tabular}{|c|c|c|c|c|c|c|c|c|c|}
		\hline
  
		\textbf{Noises}&	\multirow{2}*{\textbf{Algorithms}}&	\multicolumn{3}{|c|}{\textbf{RMSE}}  &\multicolumn{2}{|c|}{\textbf{NEES}}&\multirow{2}*{\textbf{Time (s)}}\\ 
        \cline{3-7}
		($\sigma_{w_1}$, $\sigma_{w_2}$, $\sigma_{v}$)&	~&	 \textbf{Rob. Ori. (rad)}&	 \textbf{Rob. Pos. (m)}&	 \textbf{Fea. (m)}&	 \textbf{Rob. Pose}&	 \textbf{Fea.} & ~\\
		\hline
  \multicolumn{9}{|c|}{\textbf{Environment 4 (with Plane Features)}}\\
  \hline
\multirow{3}*{(0.005,0.01,0.01)}&Std-EKF&0.0111&0.0875&0.0501&1.173&1.436&0.536\\
~&Ideal-EKF&{\color{red}\textbf{0.0110}}&{\color{red}\textbf{0.0859}}&{\color{red}\textbf{0.0486}}&1.027&1.081&0.542\\
~&Aff-EKF&{\color{red}\textbf{0.0110}}&{\color{red}\textbf{0.0859}}&{\color{blue}\textbf{0.0487}}&1.025&1.079&0.996\\\hline
\multirow{3}*{(0.005,0.01,0.02)}&Std-EKF&0.0145&0.1121&0.0712&1.371&3.281&0.570\\
~&Ideal-EKF&{\color{blue}\textbf{0.0141}}&{\color{blue}\textbf{0.1057}}&{\color{blue}\textbf{0.0658}}&1.016&1.011&0.593\\
~&Aff-EKF&{\color{red}\textbf{0.0140}}&{\color{red}\textbf{0.1056}}&{\color{red}\textbf{0.0654}}&1.018&1.020&1.077\\\hline
\multirow{3}*{(0.005,0.02,0.02)}&Std-EKF&0.0144&0.1460&0.0785&1.414&3.620&0.581\\
~&Ideal-EKF&{\color{red}\textbf{0.0140}}&{\color{red}\textbf{0.1401}}&{\color{red}\textbf{0.0732}}&1.008&0.993&0.601\\
~&Aff-EKF&{\color{blue}\textbf{0.0141}}&{\color{blue}\textbf{0.1408}}&{\color{blue}\textbf{0.0742}}&1.020&1.031&1.100\\\hline
\multicolumn{9}{|c|}{\textbf{Environment 5 (with Plane Features)}}\\
  \hline
\multirow{3}*{(0.005,0.01,0.02)}&Std-EKF&0.0206&0.3420&0.2677&1.746&2.778&0.334\\
~&Ideal-EKF&{\color{red}\textbf{0.0183}}&{\color{red}\textbf{0.2970}}&{\color{red}\textbf{0.2064}}&1.033&1.033&0.337\\
~&Aff-EKF&{\color{blue}\textbf{0.0184}}&{\color{blue}\textbf{0.2987}}&{\color{blue}\textbf{0.2091}}&1.041&1.060&0.557\\\hline
\multirow{3}*{(0.005,0.01,0.03)}&Std-EKF&0.0247&0.3999&0.3164&1.712&3.523&0.337\\
~&Ideal-EKF&{\color{red}\textbf{0.0210}}&{\color{red}\textbf{0.3261}}&{\color{red}\textbf{0.2258}}&1.010&0.999&0.345\\
~&Aff-EKF&{\color{blue}\textbf{0.0217}}&{\color{blue}\textbf{0.3293}}&{\color{blue}\textbf{0.2410}}&1.062&1.108&0.567\\\hline
\multirow{3}*{(0.005,0.02,0.03)}&Std-EKF&0.0239&0.4731&0.3397&1.758&6.707&0.337\\
~&Ideal-EKF&{\color{blue}\textbf{0.0220}}&{\color{red}\textbf{0.4183}}&{\color{blue}\textbf{0.2846}}&1.057&1.150&0.341\\
~&Aff-EKF&{\color{red}\textbf{0.0212}}&{\color{blue}\textbf{0.4217}}&{\color{red}\textbf{0.2689}}&1.076&1.114&0.563\\\hline
	\end{tabular}\label{ResultsTab_Plane}
	
\end{table*}

\section{Discussion}\label{Discussion}
In this section, we discuss the advantages of Aff-EKF by analyzing the limitations of some existing approaches, including FEJ-EKF (FEJ2-EKF) \cite{FEJ2,con2}, OC-EKF \cite{VINS1,con5,con4}, and I-EKF \cite{IEKF,con9}.


\subsection{Limitations of FEJ-EKF and OC-EKF}\label{Sec_Limit_FEJ}
To modify Std-EKF which is implemented on the atlas $\mathcal A^\eta$\footnote{The symbols, including states and Jacobians, related to Std-EKF follow the previous sections.} for the typical point feature based SLAM, FEJ-EKF \cite{con2} and OC-EKF \cite{VINS1,con5,con4} artificially select the proper evaluation points, $\textbf{X}^*_{n|n-1}$ and $\textbf{X}^*_{n-1|n-1}$, for  Jacobian matrices which are further denoted by the superscript `` $^*$ ".

In the following, we will prove that these artificial adjustments produce extra first-order linearization errors caused by the difference between such evaluation points ($\textbf{X}^*_{n|n-1}$ and $\textbf{X}^*_{n-1|n-1}$) and the linearization points ($\textbf{X}_{n|n-1}$ and $\textbf{X}_{n-1|n-1}$).

We first consider the linearization of the observation model (\ref{gen_obs}) by Taylor expansion,
\begin{equation}\label{disc_obsresid}
    \begin{array}{rl}
         \textbf{y}_{n}&=\textbf{z}_{n}-h(\textbf{X}_{n|n-1})  =\hat{\textbf{H}}_{n}^{\eta}\bm{\eta}+o(||\bm{\eta}||)\\
         &=\hat{\textbf{H}}^{\eta*}_{n}\bm{\eta}+(\hat{\textbf{H}}^\eta_{n}-\hat{\textbf{H}}^{\eta*}_{n})\bm{\eta}+o(||\bm{\eta}||),
    \end{array}
\end{equation}
where $\textbf{y}_{n+1}$ is the measurement residual, and $o(||\bm{\eta}||)$ represents the second and higher order residual with respect to $||\bm{\eta}||$.

In FEJ-EKF and OC-EKF, the quantity
\begin{equation}\label{omit_resid}
    (\hat{\textbf{H}}^{\eta}_{n}-\hat{\textbf{H}}^{\eta*}_{n})\bm{\eta}
\end{equation}
is assumed to be small enough, and is neglected in the algorithms. However, as demonstrated below, this quantity is generally a non-negligible first-order term. 

\subsubsection{FEJ-EKF} 
For FEJ-EKF, the evaluation point $\textbf{X}^*_{n|n-1}$ for $\hat{\textbf{H}}^{\eta*}_{n}$ is the first ever available estimate. 
Since the first estimates are in general different from 
the ground-truth $\textbf{X}_{n}$, we have ${\textbf{H}}^{\eta}_{n}\neq\hat{\textbf{H}}^{\eta*}_{n}$
for the general case, where ${\textbf{H}}^{\eta}_{n}$ is the Jacobian evaluated at the ground-truth.
It follows that 
$\hat{\textbf{H}}^{\eta}_{n}\nrightarrow \hat{\textbf{H}}^{\eta*}_{n},$
when $\textbf{X}_{n|n-1}\rightarrow \textbf{X}_{n}$ (i.e. $\bm{\eta}\rightarrow 0$), where $\hat{\textbf{H}}^{\eta}_{n}$ is evaluated at the linearization point. Therefore, we obtain that in general
\begin{equation}\label{lim_res}
    \lim_{\eta\rightarrow 0}\frac{(\hat{\textbf{H}}^{\eta}_{n}-\hat{\textbf{H}}^{\eta*}_{n})\bm{\eta}}{||\bm{\eta}||}\neq \textbf{0},
\end{equation}
meaning that the neglected quantity (\ref{omit_resid}) is a first-order term.

\subsubsection{OC-EKF} 
For OC-EKF, the evaluation points {$(\textbf{X}^*_{n-1|n-1},\textbf{X}^*_{n|n-1})$} for $\hat{\textbf{F}}^{\eta*}_{n-1}$ and $\hat{\textbf{H}}^{\eta*}_{n}$ are the optimal states closest to the linearization points {$(\textbf{X}_{n-1|n-1},\textbf{X}_{n|n-1})$} under the restriction
\begin{equation}\label{oc_eq1}
    \hat{\textbf{H}}^{\eta*}_{n}\hat{\textbf{F}}^{\eta*}_{n-1,0}{\textbf{N}}^{\eta}(\textbf{X}^*_{0|0})=\textbf{0},
\end{equation}
where $\hat{\textbf{F}}^{\eta*}_{n-1,0}$ is the transition matrix in (\ref{Ob_Matrix}) evaluated at $\{\textbf{X}^*_{0|0},\cdots,\textbf{X}^*_{n|n-1}\}$, and ${\textbf{N}}^{\eta}(\textbf{X}^*_{0|0})$ is the correct unobservable subspace at the $0$-th step state $\textbf{X}^*_{0|0}$. 

For the considered situations where Std-EKF does not fulfill the observability constraint, we have in general
\begin{equation}\label{oc_eq2}
    {\textbf{H}}^{\eta}_{n}\hat{\textbf{F}}^{\eta*}_{n-1,0}{\textbf{N}}^{\eta}(\textbf{X}^*_{0|0})\neq\textbf{0}.
\end{equation}
Subtract two formulas (\ref{oc_eq1})(\ref{oc_eq2}), and take the limit, we obtain
\begin{equation}
    (\lim_{\eta\rightarrow 0}({\textbf{H}}^{\eta}_{n}-\hat{\textbf{H}}^{\eta*}_{n}))(\lim_{\eta\rightarrow 0}(\hat{\textbf{F}}^{\eta*}_{n-1,0}{\textbf{N}}^{\eta}(\textbf{X}^*_{0|0})))\neq \textbf{0}.
\end{equation}
Also, due to $\hat{\textbf{H}}^{\eta}_{n}\rightarrow{\textbf{H}}^{\eta}_{n}$ ($\eta\rightarrow 0$),
we have
\begin{equation}
    \lim_{\eta\rightarrow 0}(\hat{\textbf{H}}^{\eta}_{n}-\hat{\textbf{H}}^{\eta*}_{n})\neq \textbf{0}.
\end{equation}
Therefore, we can readily obtain the same formula (\ref{lim_res}), proving that the omitted term (\ref{omit_resid}) in OC-EKF is first-order.

In addition, by a similar analysis, we can also find that FEJ-EKF and OC-EKF neglect a first-order term for the linearization of process model (\ref{gen_mot}),
\begin{equation}\label{omit_resid2}
    (\hat{\textbf{F}}^{\eta}_{n-1}-\hat{\textbf{F}}^{\eta*}_{n-1})\bm{\eta}.
\end{equation}
Similar to the linearization in observation model, these neglected terms can cause additional first-order linearization errors in FEJ-EKF and OC-EKF, which should not be ignored by the EKF framework and thus limit the performance of the estimators.

\subsubsection{FEJ2-EKF}\label{Sec_Dis_FEJ2}
FEJ2-EKF \cite{FEJ2} is an improvement of FEJ-EKF. By QR factorization, one can obtain the left nullspace $\textbf{B}_{n}$ of the error Jacobian $\hat{\textbf{H}}^{\eta}_{n}-\hat{\textbf{H}}^{\eta*}_{n}$. 
Then, FEJ2-EKF eliminates the problematic first-order term (\ref{omit_resid}) in (\ref{disc_obsresid}) by left multiplication of $\textbf{B}_{n}$ to the observation residual (\ref{disc_obsresid}). The modified observation residual adopted in FEJ2-EKF is
\begin{equation}\label{FEJ2_modify}
    \Tilde{\textbf{y}}_{n}\triangleq\textbf{B}_{n}{\textbf{y}}_{n}=\textbf{B}_{n}\hat{\textbf{H}}^{\eta*}_{n}\bm{\eta}+o(||\bm{\eta}||).
\end{equation}

However, due to the singularity of the left nullspace $\textbf{B}_{n}$, the nullspace operation (\ref{FEJ2_modify}) could inevitably give rise to some geometrical information loss. If the first estimates are good enough (very close to the ground-truth), this geometrical information loss may lead to negative effects. 
Also, as discussed in \cite{FEJ2}, numerical
instability may occur when employing QR factorization. 

In addition, FEJ2-EKF only modifies the linearization of observation functions, and still neglects some additional first-order linearization errors (\ref{omit_resid2}) in the propagation. 

Therefore, FEJ2-EKF overcomes part of the limitations in FEJ-EKF. But it produces other potential
issues.


\subsection{Advantages of Aff-EKF framework over I-EKF framework}
 Comparing to FEJ-EKF (FEJ2-EKF) and OC-EKF, I-EKF as well as Aff-EKF are able to naturally maintain the correct observability property without artificially adjusting the evaluation points of Jacobians. And thus the limitations of FEJ-EKF and OC-EKF described in Sec. \ref{Sec_Limit_FEJ} do not exist in I-EKF and our proposed Aff-EKF frameworks. Therefore, I-EKF and Aff-EKF frameworks are the natural solutions to overcome the inconsistency issue.
 
\subsubsection{Limitations of I-EKF}
As presented in \cite{IEKF,con9}, consistent I-EKF is based on a well-designed Lie group structure for a specific dynamic system. 
This Lie group should be compatible with the intrinsic structure of the considered system. 
However, these studies on I-EKF presuppose that we have found such a suitable Lie group. It is unclear about the general approach to construct these Lie group structures. 

In the applications of I-EKF \cite{RIEKF_2DSLAM,con9,InGVIO,Obj1,RIEKF_VINS,con3}, the proper Lie group structures are elaborately constructed based on experiences. Finding a proper Lie group structure for a dynamic system can be challenging and may limit the application of I-EKF framework. Because there could exist many Lie group structures on the state space manifold, while I-EKF only fulfills the observability constraint based on the specific Lie group structures. 

For example, \citet{RIEKF_2DSLAM} propose an I-EKF algorithm on a special Lie group $\mathbb{SE}_{K+1}(2)$ for 2D SLAM with point features. 
However, we should note that Std-EKF is also based on a product Lie group $\mathbb{SO}(2)\times \mathbb{R}^{2K+2}$, where $\mathbb{SO}(2)$ with matrix multiplication is the 2D rotation group and $\mathbb{R}^{2K+2}$ with vector plus. It shows that only specific Lie groups have the ability to address the inconsistency issue.


\subsubsection{Relation between Aff-EKF and I-EKF}\label{Dis_RI2}
In contrast, we provide a clear procedure (Alg. \ref{Alg3:FindA}) to design an observability preserved Aff-EKF without finding any specific Lie group structures. 

In the following, we will show that by selecting a proper atlas, Aff-EKF is able to share the same observability property as any given I-EKF.

Let's take the RI-EKF as an example\footnote{The same conclusion can be obtained by the following analysis for the left invariant EKF.}. Suppose there is an RI-EKF based on a Lie group structure, $\mathcal G=(\mathcal M, \boxplus)$, of the state space $\mathcal M$, where $\boxplus$ represents the group action. Then the atlas, $\mathcal A^\gamma=\{\pi_\textbf{X}\}$, adopted in RI-EKF is created by the retractions from the right invariant tangent space to the state space $\mathcal{M}$ \cite{Manifoldbook,IEKF,Hall,Lee2012}, i.e.
\begin{equation}
    \bm{\gamma}=\pi_{\hat{\textbf{X}}}(\textbf{X})=\log^{\mathcal G}(\textbf{X} \boxplus \hat{\textbf{X}}^{-1}),
\end{equation}
where $\bm{\gamma}$ is the error state, and $\log^{\mathcal G}$ is the inverse of exponential mapping which is unique for a given Lie group structure \cite{Hall,Lee2012}. 

Let the standard atlas be the atlas $\mathcal A^\eta=\{\phi_{\textbf{X}}\}$ where Std-EKF is implemented. Considering the transformation between error vectors in RI-EKF and Std-EKF, we have
\begin{equation}\label{err_trans_ri}
    \bm{\gamma}=\pi_{\hat{\textbf{X}}}\circ\phi^{-1}_{\hat{\textbf{X}}}(\bm{\eta}).
\end{equation}
By the Taylor expansion at the origin $\bm{\eta}=0$, the above transformation (\ref{err_trans_ri}) becomes
\begin{equation}
    \bm{\gamma}=(\frac{\partial \pi_{\hat{\textbf{X}}}\circ \phi_{\hat{\textbf{X}}}^{-1}}{\partial \bm{\eta}}|_{\bm{\eta}=\textbf{0}})\bm{\eta}+ o(||\bm{\eta}||).
\end{equation}
 Then, we can design an affine atlas $\mathcal A^\xi=\{\psi_{\textbf{X}}=\textbf{A}_{\textbf{X}}\cdot \phi_{\textbf{X}}|\textbf{X}\in \mathcal M\}$ by letting $\textbf{A}_{\textbf{X}}$ be 
\begin{equation}\label{A_ri2aff}
    \textbf{A}_{\textbf{X}}=\frac{\partial \pi_{{\textbf{X}}}\circ \phi_{{\textbf{X}}}^{-1}}{\partial \bm{\eta}}|_{\bm{\eta}=\textbf{0}}.
\end{equation}
And denote $\bm{\xi}$ as the error state in the corresponding Aff-EKF, we have
\begin{equation}\label{dis_ri_eq1}
    \bm{\gamma}=\bm{\xi}+ o(||\bm{\xi}||),
\end{equation}
where the infinitesimal $o(||\bm{\xi}||)$ will produce a slight difference for the state update in the EKF framework. By the manifold theory \cite{Manifoldbook} and the similar derivations of Eq. (\ref{Jaco_trans}), it is easy to verify that such an Aff-EKF has the same Jacobian matrices with RI-EKF. 

As a result, we can obtain an Aff-EKF through (\ref{A_ri2aff}) that has the same observability property and similar estimation as RI-EKF. And we cannot theoretically prove the superiority or inferiority in accuracy between them. Because the only differences between them in each step of estimation are the second and higher order small quantities, while both of them have neglected some other second and higher order small quantities during linearizations. The comparison of their performances may rely on the actual noise distributions.

In summary, our Aff-EKF framework shows that, in order to design a naturally consistent EKF, we do not have to struggle with specific Lie group structures that are compatible with the intrinsic structure of the system. Our proposed Aff-EKF framework provides a clear procedure to design naturally consistent EKF solvers for systems, including constrained point based SLAM and plane based SLAM, where I-EKF may not be readily available.

\section{Conclusion}

In this paper, we first prove some sufficient and
necessary conditions for an observability preserved EKF from the perspective of the dependence of unobservable subspace on the state values. Based on these newly proved theorems, the Aff-EKF framework is proposed to address the inconsistency issue in EKF for the state estimation problems, which is a theoretically sound methodology with a clear design procedure. As demonstrations, we presented our Aff-EKFs for three SLAM applications, with typical point features, plane constrained point features and plane features. 

As shown in the application of our framework in point feature based SLAM, we could have multiple consistent Aff-EKF estimators for one problem. In the further research, we will investigate the rules to select the best one in terms of accuracy and efficiency.
Moreover, we will also extend our framework to the state estimation of continuous dynamic systems, and apply it to other problems such as estimations of visual-inertial navigation systems and SLAM in dynamic/deformable environments.

\begin{appendices}
\section{Proof of Lemma \ref{equRank}}\label{equRank_proof}
\begin{proof}
Since $\mathcal A^\eta=\{\phi_{\hat{\textbf{X}}}\}$, $\mathcal A^\xi=\{\psi_{\hat{\textbf{X}}}\}$ are both the atlases of m-dimensional manifold $\mathcal M$, they are compatible due to the properties of manifold \cite [Proposition 1.17] {Lee2012}. Therefore, the Jacobian matrix of coordinate transformation 
$$\textbf{A}_{\hat{\textbf{X}}}\triangleq\frac{\partial \psi_{\hat{\textbf{X}}}\circ\phi^{-1}_{\hat{\textbf{X}}}}{\partial \bm{\eta}}|_{\bm{\eta}=\textbf{0}}$$
is an invertible matrix of dimension $m\times m$.
Then, due to 
$$\begin{array}{rlll}
{\textbf{F}}_{i-1}^{\xi}&=\textbf{A}_{\textbf{X}_{i}}{\textbf{F}}_{i-1}^{\eta}\textbf{A}^{-1}_{\textbf{X}_{i-1}},\\
{\textbf{H}}_{i}^{\xi}&={\textbf{H}}_{i}^{\eta}\textbf{A}_{\textbf{X}_{i}}^{-1},\\
    \end{array}$$
we have 
\begin{equation}
\textbf{O}^\eta_k=\textbf{O}^\xi_k\cdot \textbf{A}_{\textbf{X}_{0}}, \ \forall k\geq 0.
\end{equation}
Since $\textbf{A}_{\textbf{X}_{0}}$ is invertible, we can obtain that
\begin{equation}
    \text{rank}(\textbf{O}^\xi_k)=\text{rank}(\textbf{O}^\xi_k\textbf{A}_{\textbf{X}_{0}})=\text{rank}(\textbf{O}^\eta_k), \ \forall k\geq 0,
\end{equation}
and then,
\begin{equation}
    \text{dim}(\textbf{N}^\xi_k)=m-\text{rank}(\textbf{O}^\xi_k)=\text{dim}(\textbf{N}^\eta_k), \ \forall k\geq 0.
\end{equation}
\end{proof}



\section{Proof of Theorem \ref{MT1}}\label{MT1_proof}

\begin{proof}
Due to Eq. (\ref{EKF_OC}) and Assumption \ref{asmp2} and \ref{asmp3}, we have 
\begin{equation}\label{pe1}
    \begin{array}{rl}
        \text{dim}(\hat{\textbf{N}}^\epsilon_k(\textbf{X}_{0|0})) &=\text{dim}(\textbf{N}^\epsilon_k(\textbf{X}_{0}))  \\
         & =\text{dim}(\textbf{N}^\epsilon_1(\textbf{X}_{0}))\\
         &=\text{dim}(\hat{\textbf{N}}^\epsilon_1(\textbf{X}_{0|0})),
    \end{array}
\end{equation}
$\forall k\geq 1$ and $\forall \textbf{X}_{0},\textbf{X}_{0|0}\in \mathcal M$. If $\textbf{X}_{0}=\textbf{X}_{0|0}$ and $\textbf{X}_{1}=\textbf{X}_{1|0}$, there will be ${\textbf{N}}^\epsilon_1(\textbf{X}_{0|0})= \hat{\textbf{N}}^\epsilon_1(\textbf{X}_{0|0}), \forall \textbf{X}_{0|0}$. 
Then, due to $\hat{\textbf{N}}^\epsilon_k\subset \hat{\textbf{N}}^\epsilon_1$ and ${\textbf{N}}^\epsilon_k\subset {\textbf{N}}^\epsilon_1$, we can obtain that 
\begin{equation}
    \hat{\textbf{N}}^\epsilon_k(\textbf{X}_{0|0})= \hat{\textbf{N}}^\epsilon_1(\textbf{X}_{0|0})={\textbf{N}}^\epsilon_1(\textbf{X}_{0|0})= {\textbf{N}}^\epsilon_k(\textbf{X}_{0|0}), \forall \textbf{X}_{0|0}.
\end{equation}

Further, in order to prove the theorem, we only need to prove that the $1$-order observability matrix $\hat{\textbf{N}}^\epsilon_1$ of EKF model is a constant matrix, which is unrelated to the state values. 

Now, let's consider the following $2$-order observability matrix of EKF,
\begin{equation}\label{Ap_B_eq2}
    {\hat{\textbf{O}}}^\epsilon_{2}(\textbf{X}_{0|0})=\left[
		\begin{array}{cccccc}
			{\textbf{H}}(\textbf{X}_{0|0})\\
                {\textbf{H}}(\textbf{X}_{1|0}){\textbf{F}}(\textbf{X}_{1|0},\textbf{X}_{0|0})\\
                {\textbf{H}}(\textbf{X}_{2|1}){\textbf{F}}(\textbf{X}_{2|1},\textbf{X}_{1|1}){\textbf{F}}(\textbf{X}_{1|0},\textbf{X}_{0|0})
		\end{array}
		\right],
\end{equation}
we have
\begin{equation}\label{appendBEq1}
\begin{array}{cc}
    {\textbf{F}}(\textbf{X}_{1|0},\textbf{X}_{0|0})\hat{\textbf{N}}^\epsilon_1(\textbf{X}_{0|0})={\textbf{F}}(\textbf{X}_{1|0},\textbf{X}_{0|0})\hat{\textbf{N}}^\epsilon_2(\textbf{X}_{0|0})\\ \subset \text{null}(\left[
		\begin{array}{cccccc}
		
                {\textbf{H}}(\textbf{X}_{1|0})\\
                {\textbf{H}}(\textbf{X}_{2|1}){\textbf{F}}(\textbf{X}_{2|1},\textbf{X}_{1|1})
		\end{array}
		\right]).\end{array}
\end{equation}
If Condition (i) or (ii) in Theorem \ref{MT1} is valid, we can get that 
\begin{equation}\label{appendBEq2}
    \text{null}(\left[
		\begin{array}{cccccc}
		
                {\textbf{H}}(\textbf{X}_{1|0})\\
                {\textbf{H}}(\textbf{X}_{2|1}){\textbf{F}}(\textbf{X}_{2|1},\textbf{X}_{1|1})
		\end{array}
		\right])\subset \hat{\textbf{N}}^\epsilon_1(\textbf{X}_{1|1}).
\end{equation}
And since ${\textbf{F}}(\textbf{X}_{1|0},\textbf{X}_{0|0})$ is invertible by Assumption \ref{asmp1}, through (\ref{appendBEq1}) and (\ref{appendBEq2}), we can obtain 
\begin{equation}\label{pe4}
    {\textbf{F}}(\textbf{X}_{1|0},\textbf{X}_{0|0})\hat{\textbf{N}}^\epsilon_1(\textbf{X}_{0|0})= \hat{\textbf{N}}^\epsilon_1(\textbf{X}_{1|1}).
\end{equation}
Notice that the left-hand side of Eq. (\ref{pe4}) is only related to the variables $\textbf{X}_{0|0}$ and $\textbf{X}_{1|0}$, while the right-hand side of Eq. (\ref{pe4}) is only related to the variable $\textbf{X}_{1|1}$ which can be different from $\textbf{X}_{0|0}$ and $\textbf{X}_{1|0}$. Therefore, $\hat{\textbf{N}}^\epsilon_1(\textbf{X}_{1|1})$ is a constant space unrelated to the state values. 
Further, we can conclude that $\forall \textbf{X}_{0|0},\textbf{X}_0 \in \mathcal M$,
$$\hat{\textbf{N}}_k^\epsilon(\textbf{X}_{0|0})=\textbf{N}_k^\epsilon(\textbf{X}_0)=\bar{\textbf{N}}^\epsilon, k\geq1,$$
where $\bar{\textbf{N}}^\epsilon$ is a constant space.

\end{proof}

\section{Proof of Theorem \ref{MT2}}\label{MT2_proof}
\begin{proof}
First, we prove that the $1$-order unobservable subspace of EKF model is the same as the underlying system. Let's consider another trajectory where $\textbf{X}_0=\textbf{X}_{0|0}$ and $\textbf{X}_1=\textbf{X}_{1|0}$, the unobservable subspace of underlying system should also be the same constant matrix due to the conditions in the theorem. Then, for the 1-order unobservable subspace of EKF model, we can easily obtain that
\begin{equation}{\hat{\textbf{N}}}^\epsilon_1(\textbf{X}_{0|0})=\text{null}(\left[
		\begin{array}{cccccc}
			\textbf{H}^\epsilon(\textbf{X}_{0|0})\\ \textbf{H}^\epsilon(\textbf{X}_{1|0})\textbf{F}^\epsilon(\textbf{X}_{1|0},\textbf{X}_{0|0})
		\end{array}
		\right])=\bar{\textbf{N}}^\epsilon.
\end{equation}

Next, we will prove that the $k$-order unobservable subspace of EKF model is also the same as the underlying system.

Due to the conditions that $\forall \textbf{X}_0\in \mathcal M$, and $k\geq 1$,
$$\textbf{N}^\epsilon_k(\textbf{X}_0)=\textbf{N}^\epsilon_1(\textbf{X}_0)\equiv\bar{\textbf{N}}^\epsilon,$$
we have 
$$\bar{\textbf{N}}^\epsilon=\textbf{N}^\epsilon_2(\textbf{X}_{1})\subset \text{null}(\left[
		\begin{array}{cccccc}
\textbf{H}^\epsilon_1\textbf{F}^\epsilon_0\\ \textbf{H}^\epsilon_2\textbf{F}^\epsilon_1\textbf{F}^\epsilon_0
		\end{array}
		\right])=\textbf{F}^\epsilon_0\textbf{N}^\epsilon_1(\textbf{X}_{1})=\textbf{F}^\epsilon_0\bar{\textbf{N}}^\epsilon,$$
which follows ${\textbf{F}}^\epsilon_0\bar{\textbf{N}}^\epsilon= \bar{\textbf{N}}^\epsilon$ by Assumption \ref{asmp2}.

Let's consider another 1-step trajectory starting at $\textbf{X}_{i|i}$ and ending at $\textbf{X}_{i+1|i}$, we can get $\hat{\textbf{F}}^\epsilon_i\bar{\textbf{N}}^\epsilon= \bar{\textbf{N}}^\epsilon$, where $\hat{\textbf{F}}^\epsilon_i=\textbf{F}(\textbf{X}_{i+1|i},\textbf{X}_{i|i})$. Further, for the transition matrix $\hat{\textbf{F}}^\epsilon_{k,0}=\hat{\textbf{F}}^\epsilon_k\cdots\hat{\textbf{F}}^\epsilon_0$ of EKF model, we can obtain 
\begin{equation}\label{proofT2_eq1}
\hat{\textbf{F}}^\epsilon_{k,0}\bar{\textbf{N}}^\epsilon= \bar{\textbf{N}}^\epsilon.
\end{equation}
It follows that
\begin{equation}
    \bar{\textbf{N}}^\epsilon \subset \text{null}(\hat{\textbf{H}}^\epsilon_{i}\hat{\textbf{F}}^\epsilon_{i-1,0}), \forall i\geq1.
\end{equation}
Then, we have
\begin{equation}
\begin{array}{rl}

    \bar{\textbf{N}}^\epsilon &\subset \text{null}(\hat{\textbf{H}}^\epsilon_{0})\bigcap^k_{i=1} \text{null}(\hat{\textbf{H}}^\epsilon_{i}\hat{\textbf{F}}^\epsilon_{i-1,0})\\&={\hat{\textbf{N}}}^\epsilon_k(\textbf{X}_{0|0})\subset {\hat{\textbf{N}}}^\epsilon_1(\textbf{X}_{0|0})=\bar{\textbf{N}}^\epsilon, k\geq 1,
    \end{array}
\end{equation}
which implies that ${\hat{\textbf{N}}}^\epsilon_k(\textbf{X}_{0|0})=\bar{\textbf{N}}^\epsilon$.

Therefore, we can conclude that, for $k\geq1$,
\begin{equation}
    \hat{\textbf{N}}^\epsilon_k(\textbf{X}_{0|0})=\bar{\textbf{N}}^\epsilon={\textbf{N}}^\epsilon_k(\textbf{X}_{0}).
\end{equation}



\end{proof}

\section{Proof of Lemma \ref{Lemma1}} \label{Lemma1_proof}
\begin{proof}
Due to 
$$\begin{array}{llll}
{\textbf{F}}_{i-1}^{\xi}&=\textbf{A}_{\textbf{X}_{i}}{\textbf{F}}_{i-1}^{\eta}\textbf{A}^{-1}_{\textbf{X}_{i-1}}, i=1,\cdots,n,\\
{\textbf{H}}_{n}^{\xi}&={\textbf{H}}_{n}^{\eta}\textbf{A}_{\textbf{X}_{n}}^{-1},\\
    \end{array}$$
we have
$${\textbf{H}}^{\xi}_{n}{\textbf{F}}_{n-1,0}^{\xi}={\textbf{H}}^{\eta}_{n}{\textbf{F}}_{n-1,0}^{\eta}\textbf{A}_{\textbf{X}_{0}}^{-1}.$$
Then we can easily get that 
$${\textbf{N}}_k^{{\xi}}(\textbf{X}_0)=\textbf{A}_{\textbf{X}_0}{\textbf{N}}_k^{{\eta}}(\textbf{X}_0).$$
\end{proof}

\end{appendices}

\bibliographystyle{plainnat}
\bibliography{references}

\begin{IEEEbiography}    [{\includegraphics[width=1in,height=1.33in]{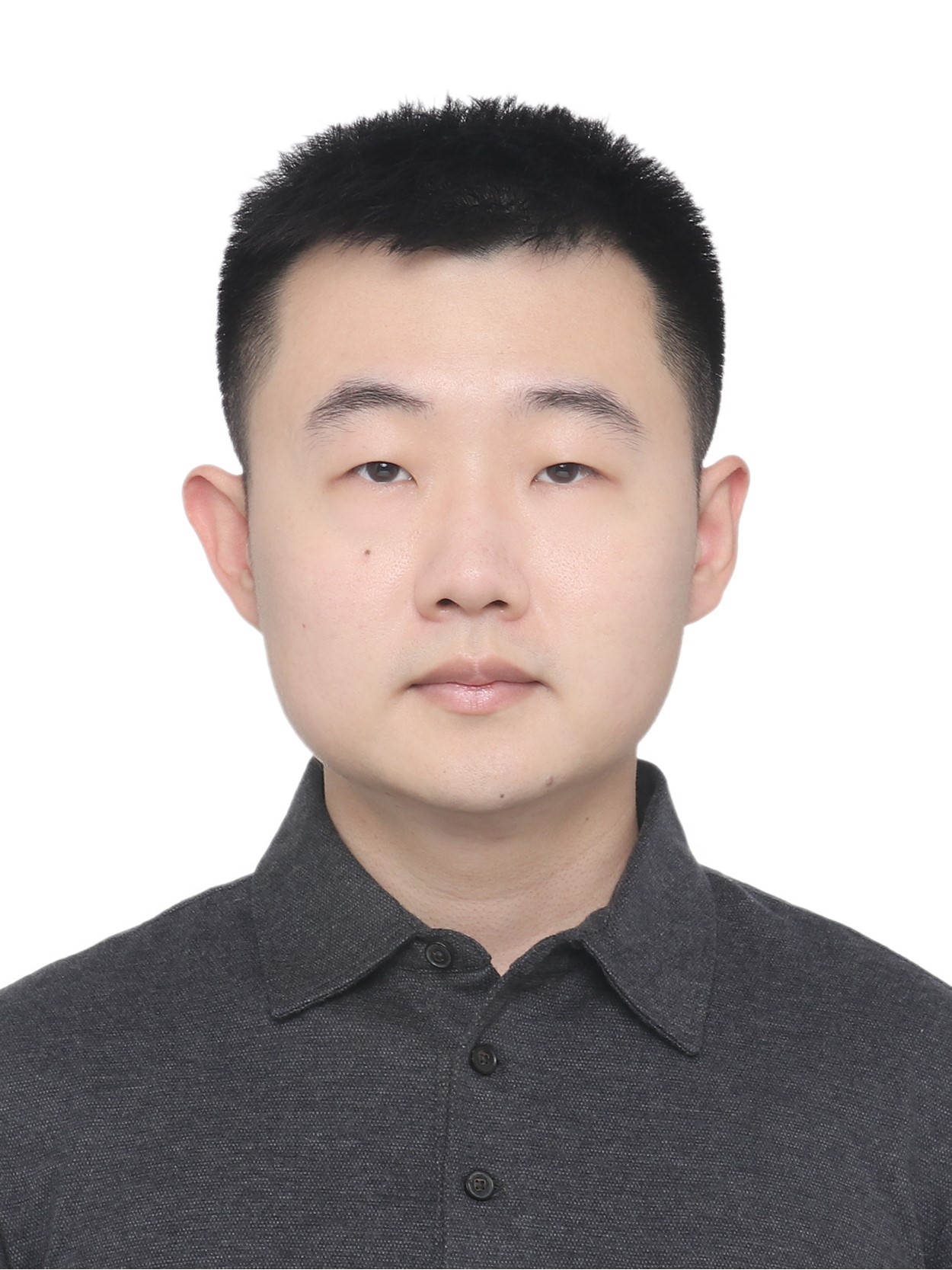}}] 
{Yang Song} received his B.S. and M.S. from School of Mathematics and Statistics, Beijing Institute of Technology in 2017 and 2020, respectively. He is currently a Ph.D. candidate in the Robotics Institute, University of Technology Sydney, Australia. His current research interests include extend Kalman filter, SLAM and optimization.
  \end{IEEEbiography}
  
\begin{IEEEbiography}    [{\includegraphics[width=1in,height=1.35in]{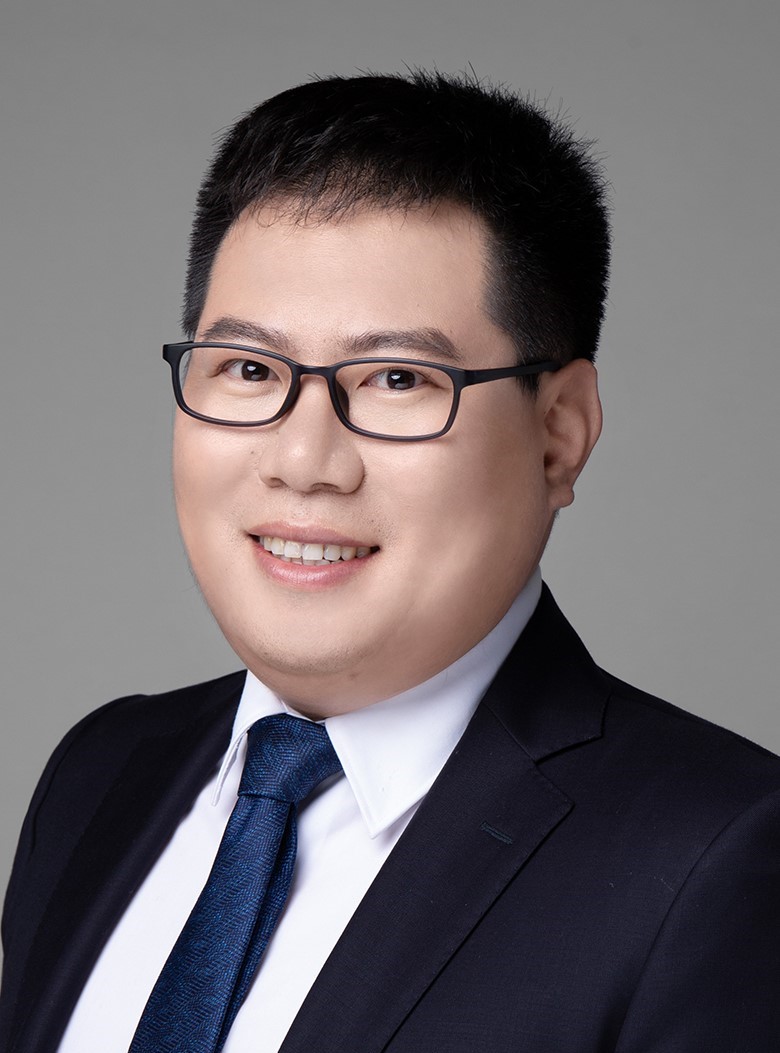}}] 
{Liang Zhao} (Member, IEEE) received the Ph.D. degree in photogrammetry and remote sensing from Peking University, Beijing China, in January 2013. From 2014 to 2016, he worked as a Postdoctoral Research Associate with the Hamlyn Centre for Robotic Surgery, Department of Computing, Faculty of Engineering, Imperial College London, London, U.K. From 2016 to 2024, he was a Senior Lecturer and the Director of Robotics in Health at the UTS Robotics Institute, Faculty of Engineering and Information Technology, University of Technology Sydney, Australia.
He is currently a Reader in Robot Systems in the School of Informatics, The University of Edinburgh, United Kingdom. His research interests include surgical robotics, robots simultaneous localisation and mapping (SLAM), monocular SLAM, aerial photogrammetry, optimisation techniques in mobile robot localisation and mapping and image guide robotic surgery. 
Dr. Zhao is an Associate Editor for IEEE TRANSACTIONS ON ROBOTICS, ICRA, and IROS.

\end{IEEEbiography}
\begin{IEEEbiography}    [{\includegraphics[width=1in,height=1.26in]{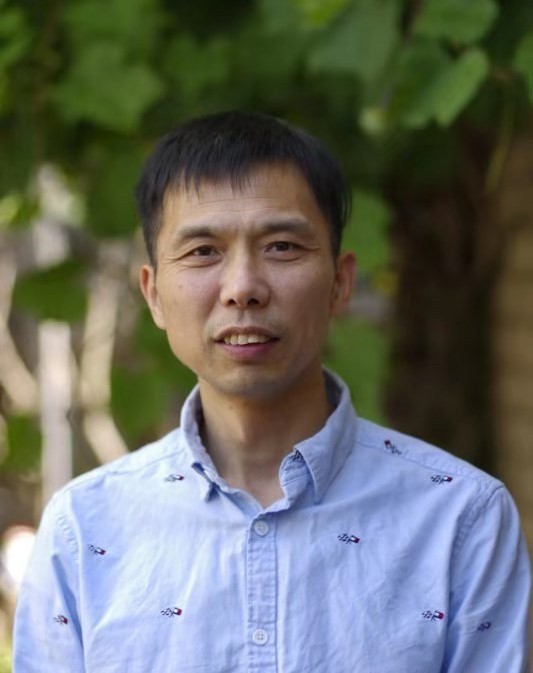}}] 
{Shoudong Huang} (Senior Member, IEEE) received the Ph.D. degree in automatic control from Northeastern University, Shenyang, China, in 1998. He is currently a Professor in Robotics Institute, Faculty of Engineering and Information Technology, University of Technology Sydney, Sydney, Australia. His current research interests include mobile robots simultaneous localization and mapping (SLAM), and robot path planning and control.
  \end{IEEEbiography}

\end{document}